%% file: paper.tex
\documentclass{article} 
\PassOptionsToPackage{numbers,sort&compress}{natbib}
\usepackage{iclr2027_conference,times}
\setcitestyle{numbers,square,citesep={,},aysep={,}}

\input{math_commands.tex}

\usepackage[utf8]{inputenc}
\usepackage[T1]{fontenc}
\usepackage{hyperref}
\usepackage{url}
\usepackage{booktabs}
\usepackage{amsfonts}
\usepackage{amssymb}
\usepackage{amsmath}
\usepackage{amsthm}
\usepackage{mathtools}
\usepackage{microtype}
\usepackage{xcolor}
\usepackage{graphicx}
\usepackage{subcaption}
\usepackage{enumitem}

\usepackage{amsmath,amsthm,amssymb}
\usepackage{algorithm,algorithmic}
\allowdisplaybreaks
\newtheorem{proposition}{Proposition}

\newcommand{\diag}{\mathrm{diag}}

\newcommand{\defeq}{:=}

\newcommand{\vQ}{\mathbf{Q}}
\newcommand{\vK}{\mathbf{K}}
\newcommand{\vV}{\mathbf{V}}

\newcommand{\vdQ}{\mathbf{dQ}}
\newcommand{\vdK}{\mathbf{dK}}
\newcommand{\vdV}{\mathbf{dV}}

\newcommand{\vS}{\mathbf{S}}
\newcommand{\vdS}{\mathbf{dS}}

\newcommand{\vP}{\mathbf{P}}
\newcommand{\vdP}{\mathbf{dP}}

\newcommand{\vX}{\mathbf{X}}
\newcommand{\vdX}{\mathbf{dX}}
\newcommand{\vO}{\mathbf{O}}
\newcommand{\vdO}{\mathbf{dO}}

\newcommand{\vm}{\mathbf{m}}
\newcommand{\vell}{\mathbf{\ell}}

\newcommand{\vL}{\mathbf{L}}

\newcommand{\fullName}{MassAlloc Attention}
\newcommand{\shortName}{MALA}

\title{\fullName{}:\\ Let Attention Allocate Its Own Compute}

\author{
  \parbox{\textwidth}{
    Jingze Shi$^{12}$\thanks{Equal contribution. $^1$The Hong Kong University of Science and Technology (Guangzhou), Guangzhou, China. $^2$Beijing Academy of Artificial Intelligence, Beijing, China. $^3$Université Paris Cité, Paris, France. Corresponding author: Guang Liu <liuguang@baai.ac.cn>, Yuyu Luo <yuyuluo@hkust-gz.edu.cn>.} \ \ Zhangyang Peng$^{1}$ \ \ Xianduo Li$^{2}$ \ \ Yanlin Qi$^{23}$ \ \ Xiaotian Lin$^{1}$ \ \ Haoxian Chen$^{12}$  \\[2pt]
    Liangdong Wang$^{2}$ \ \ Guang Liu$^{2}$ \ \ Yuyu Luo$^{1}$
  }
}

\iclrfinalcopy 
\begin{document}

\maketitle
\fancyhead{}
\lhead{Let Attention Allocate Its Own Compute}

\begin{abstract}
Long-context full softmax attention (FullAttn) often assigns negligible normalized mass to much of the causal score space, yet dense kernels execute the complete post-score path after forming each QK tile. We introduce \textbf{\fullName{} (\shortName{})}, a fused attention primitive that preserves score access to every legal causal interaction and uses normalized contribution to allocate post-score computation. Forward uses its evolving online-softmax normalizer, while backward reuses the finalized normalizer to derive nested retained support using only standard attention state. A common tolerance governs training and inference, allowing the retained work to adapt across queries, heads, layers, and inputs. \shortName{} reduces low-contribution post-score computation while retaining quadratic QK score discovery. A matched-work study at 8K isolates the benefit of distribution-adaptive allocation: under exactly matched total post-score work, \shortName{} approaches a per-instance reference-mass oracle, with mean omitted mass of 0.0188\% versus 0.0182\%, while static allocations perform substantially worse. Across context lengths from 1K to 32K tokens, the same tolerance maintains low output and gradient errors relative to the reference operator. Across a broader controlled associative-recall comparison, \shortName{} closely tracks FullAttn as context grows, reaching 89.67\% accuracy at 8K compared with 89.97\% for FullAttn. In an attention-operator benchmark at 128K tokens on 8 GPUs with tensor parallelism, \shortName{} reduces forward and backward latency during training by $2.2\times$ and $3.0\times$ and decoding latency during inference by $1.6\times$ relative to FullAttn, while retaining FullAttn-level per-rank peak operator memory. Across scaling-law training from 0.6B to 14B parameters on 128 GPUs, \shortName{} closely tracks FullAttn in perplexity while reducing total training FLOPs, with a 23.1\% reduction at 14B during 32K-context training. The resulting 14B models and 32B models from separate continued training achieve comparable knowledge, reasoning, and long-context retrieval scores to FullAttn. These results indicate that allocating post-score computation according to normalized attention contributions can retain the evaluated capabilities of FullAttn while reducing attention computation.
Our code is open-sourced at \href{https://github.com/HKUSTDial/flash-sparse-attention}{\nolinkurl{flash-sparse-attention}}.
\end{abstract}

\input{sections/introduction}

\input{sections/methodology}

\input{sections/experiments}

\input{sections/related_work}

\input{sections/conclusion}



\bibliography{biblio}
\bibliographystyle{iclr2027_conference}

\appendix
\input{sections/appendix}

\end{document}

%% file: math_commands.tex
\usepackage{amsmath,amsfonts,bm}

\def\eqref#1{equation~\ref{#1}}

\def\1{\bm{1}}

\def\vo{{\bm{o}}}

\def\vv{{\bm{v}}}

\DeclareMathAlphabet{\mathsfit}{\encodingdefault}{\sfdefault}{m}{sl}
\SetMathAlphabet{\mathsfit}{bold}{\encodingdefault}{\sfdefault}{bx}{n}



%% file: sections/introduction.tex
\section{Introduction}
\label{sec:introduction}

Context lengths are expanding from thousands to hundreds of thousands of tokens~\citep{snell2024tts} to support long-document understanding~\citep{park2023generative,gemini2025}, multi-turn reasoning~\citep{hf2025openr1,deepseekai2025deepseekr1incentivizingreasoningcapability,qwen32025}, and repository-level code generation~\citep{zhang2024codeagent}.
Attention over these long contexts incurs substantial computation and memory traffic.
FlashAttention~\citep{dao2022flashattention,shah2024flashattention3} improves the IO efficiency of self-attention~\citep{vaswani2017attention} through tiling, fusion, and online softmax~\citep{milakov2018onlinesoftmax}.
However, these IO improvements leave the dense execution pattern unchanged: every legal causal tile still incurs post-score computation after its $\vQ\vK$ scores are formed, including softmax updates, $\vV$ loading, and $\vP\vV$ accumulation in forward, and probability reconstruction and the computation of $\vdP$, $\vdS$, $\vdQ$, $\vdK$, and $\vdV$ in backward.

Full softmax attention (FullAttn) is highly non-uniform: its mass often concentrates in local regions, sink tokens, and a sparse set of retrieval-relevant interactions~\citep{gu2024whenas,barbero2025whydl,queipodellano2025attentionsa,xiao2024duoattentionel}.
Many remaining interactions receive negligible normalized mass~\citep{gao2024seerattention,yuan2025blasstdynamicblockedattention}, yet dense kernels still execute the full forward and backward post-score computation for low-contribution tiles.
Attention already produces scores and softmax statistics that can guide these decisions.
These observations motivate preserving score access to the complete causal context while using normalized contribution to decide which tiles warrant further computation, allowing attention to allocate its own post-score compute.

Allocating this work also requires an execution structure that remains efficient during training and inference.
Static windows and block patterns prescribe a position-defined support~\citep{child2019generating,beltagy2020longformerlongdocumenttransformer,zaheer2020big,fu2025slidingwindowattentiontraining}, while dynamic methods select tokens or blocks through content-dependent scores or routing~\citep{tang2024quest,lai2025flexprefill,li2024snapkv,zhang2023h2o,xiao2024infllm,qi2026pariskv,zhao2025infllmv2,yuan2025nativesparseattentionhardwarealigned,lu2025moba,gao2024seerattention}.
Adaptive top-$p$ methods vary the selected support with an estimated mass target~\citep{lin2025twilightadaptiveattentionsparsity,ni2026doublephierarchicaltoppsparse}.
These approaches can avoid QK computation for excluded interactions; a complementary opportunity is to allocate post-score work inside the attention loop after forming each exact QK tile.
This raises the question: can allocating post-score computation from attention's own normalized contributions preserve operator fidelity, model quality, and long-range retrieval while reducing the cost of training and inference?

\begin{figure}[!t]
    \centering
    \includegraphics[width=\textwidth]{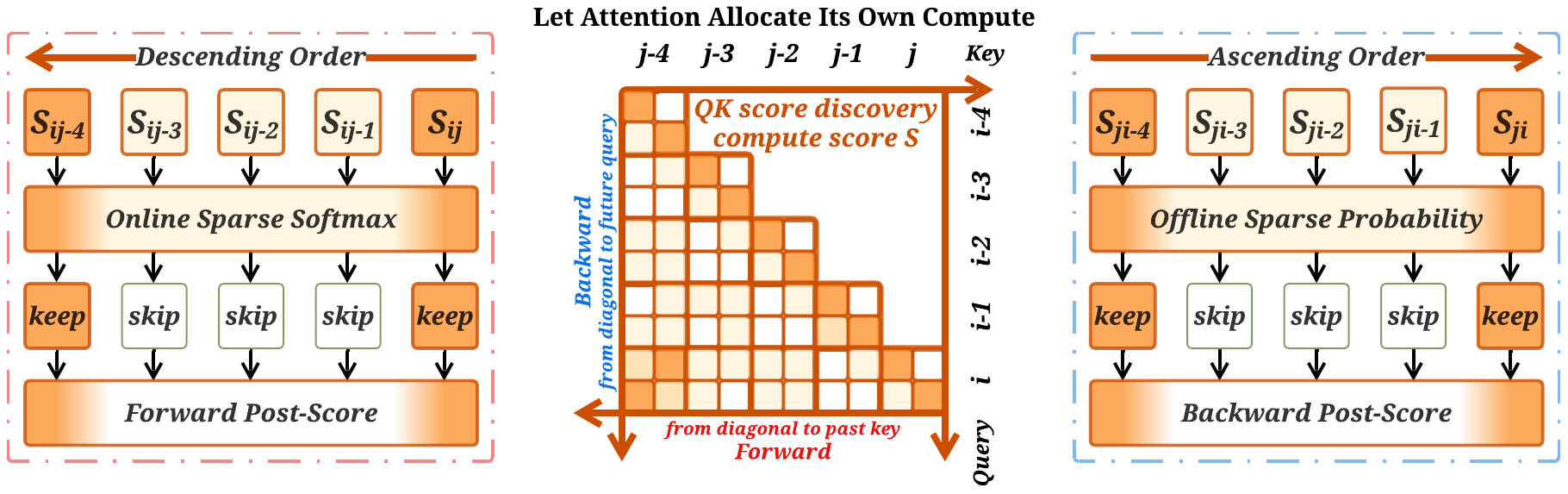}
    \caption{
    \textbf{\shortName{} lets attention allocate its own post-score compute.}
    Every legal causal tile undergoes QK score discovery. Forward uses its evolving softmax normalizer to allocate forward post-score execution; backward reuses the finalized normalizer to derive nested retained support while storing only standard attention state.
    }
    \label{fig:teaser}
\end{figure}

We introduce \textbf{\fullName{} (\shortName{})}, a fused attention primitive that converts normalized contribution into runtime compute allocation, as illustrated in Figure~\ref{fig:teaser}.
Every legal causal tile first undergoes QK score discovery.
During forward, \shortName{} bounds normalized contributions using its evolving online-softmax normalizer; during backward, it uses the finalized normalizer saved by forward to evaluate each recomputed tile.
In either pass, tiles whose contributions fall below a length-normalized tolerance omit their post-score work.
The tolerance controls admissible contribution, while the realized attention distribution determines the retained computation across queries, heads, layers, and inputs.

The same allocation rule also supports fused execution for training and inference.
Forward bounds the total dense probability mass assigned to omitted positions, while backward derives retained support nested within the forward support using only the saved normalizer.
The operators use standard attention state without materializing a selection mask or retained-tile indices.
A common tolerance governs the training forward and backward passes, inference prefill, and autoregressive decoding, while the realized work remains specific to each execution context.

\shortName{} retains quadratic QK score discovery; its savings arise from avoiding low-contribution post-score execution.
\shortName{} develops this execution direction into a normalized-mass allocation rule with paired forward and backward semantics.
The forward rule provides an omitted-mass guarantee, while output and gradient fidelity are evaluated empirically.

We evaluate \shortName{} around two questions: whether normalized-mass allocation preserves operator fidelity, language-model quality, and long-range retrieval, and whether the resulting reduction in post-score work translates into practical training and inference efficiency under tensor parallelism~\citep{megatron-lm}.
A matched-work allocation study first isolates the benefit of distribution-adaptive allocation: under exactly matched total post-score work, \shortName{} approaches the per-instance reference-mass oracle, with mean omitted mass of 0.0188\% versus 0.0182\%, while static position- and layer-head-based allocations perform substantially worse.
With the same tolerance from 1K to 32K tokens, \shortName{} maintains low relative output and gradient errors against the reference operator.
Across a broader controlled associative-recall comparison, fixed-budget sparse baselines use a 1,024-slot ceiling while \shortName{} adaptively allocates post-score work at a similar scale; at 8K, it reaches 89.67\% accuracy, compared with 89.97\% for FullAttn.
In the attention-operator benchmark at 128K sequence length on 8 H100 GPUs with tensor parallelism, \shortName{} reduces forward and backward latency during training by $2.2\times$ and $3.0\times$, and decoding latency during inference by $1.6\times$ relative to FullAttn, while retaining FullAttn-level peak operator memory.
Across scaling-law training from 0.6B to 14B parameters, \shortName{} closely tracks FullAttn in perplexity; at 14B, it reduces total training FLOPs by 2.5\% during 4K pre-training and 23.1\% during 32K long-context training.
At the model level, the resulting 14B models and separately continued-trained 32B models retain aggregate knowledge, reasoning, and long-context retrieval performance comparable to FullAttn.
Our contributions are as follows:
\begin{itemize}[leftmargin=1em]
    \item We formulate sparse attention as \textbf{distribution-conditioned compute allocation}: every causal interaction remains score-accessible, while the probability distribution realized by attention determines its post-score work.
    \item We implement \textbf{\fullName{}} for training and inference with paired online forward and offline backward rules and a common normalized-mass tolerance. The operators use only native attention state, with a forward omitted-mass bound and nested backward support.
    \item We evaluate the modeling and efficiency consequences of normalized-mass allocation through matched-work controls, operator fidelity, controlled associative recall, kernel execution, scaling-law training, and model-level evaluation.
\end{itemize}

%% file: sections/methodology.tex
\section{Methodology}
\label{sec:methodology}

We formulate attention as a runtime compute allocator by separating QK score discovery from post-score execution. We first define the allocation rule and its contribution tolerance, then present the online forward and offline backward tests, characterize their computational cost, and describe fused execution for training and inference.

\subsection{Attention as a Compute Allocator}
\label{sec:fsa_allocator}

\shortName{} computes QK scores for every legal causal tile and decides whether to execute the subsequent forward or backward computation from those scores and the available softmax state.

For notational simplicity, we describe a single query head and its associated key/value head, suppressing batch and head indices. Let $N_q$ and $N_k$ be the query and key sequence lengths, with $\vQ\in\mathbb{R}^{N_q\times d_h}$ and $\vK,\vV\in\mathbb{R}^{N_k\times d_h}$. We absorb the standard softmax scale $1/\sqrt{d_h}$ into $\vQ$ and use natural exponentials and logarithms throughout. Queries are divided into blocks $I_i$ of $B_q$ rows and keys and values into blocks $J_j$ of $B_k$ rows. Lowercase $q$ and $k$ index individual query and key positions, while $i$ and $j$ index blocks.

For each tile, score discovery forms $\vS_{ij}=\vQ_i\vK_j^\top$. Forward post-score computation comprises the softmax update, $\vV$ loading, and $\vP\vV$ accumulation; backward post-score computation comprises probability reconstruction and the computation of $\vdP$, $\vdS$, $\vdQ$, $\vdK$, and $\vdV$. Here, $\vP$ denotes attention probabilities, and $\vdX$ denotes the backpropagated quantity associated with $\vX$. FullAttn executes both stages for every legal tile, whereas \shortName{} allocates post-score work after score discovery.

To define a common contribution scale, let $\mathcal{C}_q$ be the set of causally visible keys for query row $q$ and let $L_q=|\mathcal{C}_q|$. Uniform attention assigns probability $1/L_q$ to each such key. We scale this reference by a shared dimensionless tolerance $\tau>0$, giving the length-aware threshold $\tau/L_q$. \shortName{} compares the largest per-key contribution ratio in each tile with this threshold.

With zero-based token positions, $L_q=q+1$ in equal-length causal training. In single-token autoregressive decoding, $L_0=N_k$, including the current token. For block $I_i$, $\vL_i\in\mathbb{R}^{B_q}$ collects these per-row counts over valid query rows. The normalization by $L_q$ gives the tolerance the same interpretation across query positions and sequence lengths without separately calibrated thresholds. The tolerance specifies admissible contribution rather than a compute budget: sharp attention rows can reject many tiles, whereas diffuse rows generally retain more.

\subsection{Online and Offline Allocation}
\label{sec:fsa_rule}

\textbf{Online Sparse Softmax.}
The forward pass~(Algorithm~\ref{alg:fsa_forward}) maintains the standard online-softmax row maximum $\vm_i$ and shifted normalization sum $\vell_i$ over retained tiles. For query row $q$, let $R_q^{(t)}$ contain the causal keys retained before the current tile is tested, and let $s_{qk}$ be the QK score for key $k$. The retained normalizer at this point is
\begin{equation}
    Z_q^{(t)}
    \defeq \sum_{k\in R_q^{(t)}}\exp(s_{qk})
    =\exp\bigl(m_q^{(t)}\bigr)\ell_q^{(t)}
    \label{eq:fsa_running_normalizer}
\end{equation}
Here, $m_q^{(t)}$ and $\ell_q^{(t)}$ are the entries of the running block states for row $q$. For a nonempty retained set, $m_q^{(t)}+\log\ell_q^{(t)}=\log Z_q^{(t)}$. After forming a candidate tile, its largest contribution ratio for this row is
\begin{equation}
    r_{qj}^{(t)}
    \defeq
    \frac{\exp\bigl(\max_{k\in J_j\cap\mathcal{C}_q}s_{qk}\bigr)}{Z_q^{(t)}}
    \label{eq:fsa_contribution_ratio}
\end{equation}
The denominator includes only the mass retained so far, so this ratio can exceed one and upper-bounds the largest dense-softmax probability in the candidate tile for that row.

The tile is skipped only when $r_{qj}^{(t)}<\tau/L_q$ for every valid query row with a legal entry in the tile. Taking logarithms and collecting the rows of $I_i$ gives the blockwise test used by the kernel:
\begin{equation}
    \operatorname{rowmax}(\vS_{ij})
    -\vm_i-\log\vell_i
    < \log(\tau/\vL_i)
    \label{eq:fsa_online_test}
\end{equation}
The inequality is interpreted elementwise over relevant rows. Causally masked entries have score $-\infty$, and rows with no legal key in the candidate tile do not constrain its retention. The test is conservative at tile granularity: one row that fails the skip condition retains the tile for all rows in the query block.

Before any key has been retained for a row, $\ell_q^{(t)}=0$ and $Z_q^{(t)}=0$. For a candidate tile containing a legal key for that row, the log-ratio is defined as $+\infty$, so its first legal tile is necessarily retained and every forward attention row is nonempty. For causal attention, \shortName{} visits tiles from the diagonal toward earlier keys. This traversal initializes the running state from valid recent context and typically establishes a strong normalizer early. Locality is an execution prior rather than a support assumption: all earlier causal tiles still undergo score discovery, and any distant tile with a sufficiently large score is retained.

A skipped tile leaves $\vm_i$, $\vell_i$, and $\vO_i$ unchanged; a retained tile follows the standard online-softmax update. The final output is therefore ordinary softmax attention renormalized over the retained tiles. Let $R_q$ denote the final retained set and $Z_{R,q}=\sum_{k\in R_q}\exp(s_{qk})$ its unnormalized mass. At the end of forward, the kernel overwrites the running sum buffer with the final log-normalizer, $\vell_i\leftarrow\vm_i+\log\vell_i$, whose entries are $z_q=\log Z_{R,q}$. This saved state is passed to backward. Appendix~\ref{sec:appendix:fsa_mass_bound} provides the forward approximation analysis.

\textbf{Offline Sparse Probability.}
The backward kernel~(Algorithm~\ref{alg:fsa_backward}) uses the finalized normalizer instead of replaying the forward sequence of partial normalizers. Here, \emph{offline} means that the normalizer is already finalized when a tile is tested; allocation still occurs at runtime and requires no offline calibration.

Backward applies the same contribution tolerance using $Z_{R,q}$ in place of the evolving $Z_q^{(t)}$. In the key-major traversal, the recomputed score tile is $\vS_{ji}=\vK_j\vQ_i^\top=\vS_{ij}^\top\in\mathbb{R}^{B_k\times B_q}$. Each query's saved log-normalizer and threshold are therefore broadcast along the key dimension. The resulting skip condition is
\begin{equation}
    \vS_{ji}-\vell_i^\top
    < \log(\tau/\vL_i)^\top
    \label{eq:fsa_offline_test}
\end{equation}
Here the inequality is interpreted elementwise over valid causal entries and must hold throughout the tile. Otherwise, \shortName{} reconstructs $\vP_{ji}=\exp(\vS_{ji}-\vell_i^\top)$ and executes the standard $\vdV$, $\vdP$, softmax-backward, $\vdQ$, and $\vdK$ path. Unlike forward, the backward rule does not require the diagonal tile to be retained.

In exact arithmetic, $Z_{R,q}\ge Z_q^{(t)}$ for every online test with positive retained mass, so using the finalized denominator cannot increase any key's contribution ratio. With the same tile partition and causal mask in both passes, every forward-skipped tile therefore also satisfies the backward skip condition. Backward may additionally omit a tile that forward retained before its normalizer was complete. Its retained support is thus nested within the forward support without a stored forward mask.

This nesting does not imply exact differentiation of the forward operator, since backward can omit additional gradient contributions. Their error also depends on $\vdO$ and $\vV$, so we evaluate gradient fidelity directly at the operator level.

\subsection{Attention Cost}
\label{sec:fsa_cost}

For fixed head configuration, \shortName{} remains quadratic in sequence length because it computes QK scores for every legal causal tile. Its savings are data-dependent constant-factor reductions in post-score arithmetic and memory traffic: a forward skip saves full-tile exponentiation, $\vV$ loading, and $\vP\vV$ accumulation; a backward skip saves probability reconstruction and the matrix multiplications and elementwise operations used to form $\vdV$, $\vdP$, $\vdS$, $\vdQ$, and $\vdK$. The tests require a maximum reduction and scalar comparisons using existing attention state. Support-sparse methods can additionally avoid QK computation for excluded regions, whereas \shortName{} retains full score accessibility. As $\tau\rightarrow0$, the skip conditions vanish and both passes recover FullAttn.

We measure allocated work as the number of key slots per query for which the post-score path is executed. Retaining a $B_q\times B_k$ tile contributes $B_k$ post-score key slots to each of its $B_q$ query rows. We report the average as \emph{mean post-score key slots per query}, excluding QK score discovery. This quantity captures realized work across training and inference, including the effect of tile sharing across rows. When attention architectures use different value dimensions, we separately use post-score FLOP-equivalent work for cross-architecture matching; Appendix~\ref{sec:appendix:work_accounting} defines this accounting.

\subsection{Fused Execution}
\label{sec:fsa_execution}

The allocation tests are embedded in the ordinary tiled attention loop. The running maximum, shifted normalization sum, and output accumulator remain on chip during forward. Forward writes only the output $\vO$ and the final log-normalizer $\vell$, while backward reuses $\vell$ to reconstruct retained probabilities. The fused operators materialize neither the attention matrix nor a binary mask, retained-tile indices, or router state. The same forward operator serves training and inference prefill; Appendix~\ref{sec:appendix:fused_execution} details the traversal and state handling for each pass.

The same probability tolerance governs training forward and backward, inference prefill, and decoding. Under split-KV decoding, each split tests tiles against its own partial normalizer, while $L_q$ in $\tau/L_q$ still counts the query's full causal context. A smaller available normalizer makes the test more conservative and may retain additional post-score work, so a shared tolerance does not require split and unsplit execution to retain identical supports. Partial outputs are combined through the standard normalizer-aware reduction. \shortName{} neither compresses nor evicts the persistent KV cache, so its decoding-memory scope is the attention operator's working set rather than cache capacity.

%% file: sections/experiments.tex
\section{Experiments}
\label{sec:experiments}

\textbf{Evaluation Overview.}
We organize the experiments around a sequence of increasingly comprehensive questions: whether the realized attention distribution improves compute allocation under fixed post-score work, whether one tolerance preserves the reference operator across context lengths, whether the resulting allocation supports arbitrary long-range associations and efficient operator execution, and whether its quality-compute trade-off persists across model scales. We conclude with knowledge, reasoning, and long-context evaluations of the resulting 14B models and of 32B models trained in a separate continued-training study.

\textbf{Experimental Settings.}
Unless stated otherwise, \shortName{} uses the same canonical tolerance, $\tau=1$, across forward, backward, prefill, and decoding; the realized distribution sets the retained work for each layer, head, input, and sequence length. Within each experiment, attention variants use matched model scale, depth, hidden size, data, and optimization settings while retaining their native attention configurations. The matched-work and operator-fidelity studies use the same 14B \shortName{} checkpoint obtained after the 32K long-context stage of the scaling-law study. Scaling-law and 32B continued training use 128 NVIDIA H100 GPUs; model-level evaluation and operator benchmarking use 8 H100 GPUs. Appendix~\ref{sec:appendix:experiment_setup} provides the work and cost accounting, baseline configurations, training schedules, hardware settings, and evaluation protocols.

\textbf{Normalized-Mass Allocation at Matched Work.}
We isolate the value of allocating work from the realized attention distribution while holding total work fixed. At the evaluated 8K context length, all policies in Table~\ref{tab:matched_work} execute exactly the same total post-score work; their shared average of approximately 1,024 post-score key slots per query is induced by \shortName{} under $\tau=1$. We compare \shortName{}'s online decisions with three diagnostic reference-mass controls: a position-only allocation, a static layer-head-position allocation, and a per-instance allocation given the work realized by \shortName{} for each decision. All three controls rank candidate regions using finalized per-instance reference mass; they differ only in how much work they assign to each decision. Static layer-head specialization improves over position-only allocation by capturing fixed specialization patterns. Compared with static layer-head-position allocation, per-instance allocation reduces mean omitted mass by $9.5\times$ and mean relative output $L_2$ error by $17.2\times$. Using only its evolving online state, \shortName{} closely approaches this reference-mass oracle: mean omitted mass is 0.0188\% versus 0.0182\%, and mean relative output $L_2$ error is 0.0174\% versus 0.0164\%. At fixed post-score work, the reference-mass controls quantify the benefit of per-instance allocation, and \shortName{}'s online decisions recover nearly all of that benefit. Appendix~\ref{sec:appendix:matched_work_setup} details the control construction and exact work matching.

\begin{table}[!t]
    \centering
    \small
    \caption{
    \textbf{Matched-work allocation.}
    Every policy executes exactly the same total post-score work, with a shared average of approximately 1,024 slots per query. The first three rows use finalized per-instance reference mass for region selection and differ only in how work is allocated across decisions. Omitted probability mass and relative output $L_2$ error are reported as percentages; statistics pool input-layer-head-query-block allocation decisions.
    }
    \vspace{-1em}
    \resizebox{\linewidth}{!}{
    \begin{tabular}{@{}lcccccc@{}}
    \toprule
    \sc{Allocation Scope} & \sc{Selection Signal} & \sc{Mean Slots} & \multicolumn{2}{c}{\sc{Omitted Mass (\%)}} & \multicolumn{2}{c}{\sc{Output Error (\%)}} \\
    \cmidrule(lr){4-5}\cmidrule(lr){6-7}
    & & & \sc{Mean $\downarrow$} & \sc{P95 $\downarrow$} & \sc{Mean $\downarrow$} & \sc{P95 $\downarrow$} \\
    \midrule
    Position & Final reference mass & $\approx$1,024 & $0.6012$ & $0.9923$ & $0.6231$ & $4.237\phantom {0}$ \\
    Layer/Head/Position & Final reference mass & $\approx$1,024 & $0.1721$ & $0.6247$ & $0.2817$ & $0.9237$ \\
    Input/Layer/Head/Position & Final reference mass & $\approx$1,024 & $\mathbf{0.0182}$ & $\mathbf{0.0664}$ & $\mathbf{0.0164}$ & $\mathbf{0.1013}$ \\
    \textbf{Input/Layer/Head/Position} & \textbf{\shortName{} online mass bound} & $\approx$1,024 & $\underline{0.0188}$ & $\underline{0.0693}$ & $\underline{0.0174}$ & $\underline{0.1017}$ \\
    \bottomrule
    \end{tabular}
    }
    \vspace{-1.0em}
    \label{tab:matched_work}
\end{table}

\begin{figure}[!t]
    \centering
    \includegraphics[width=\textwidth]{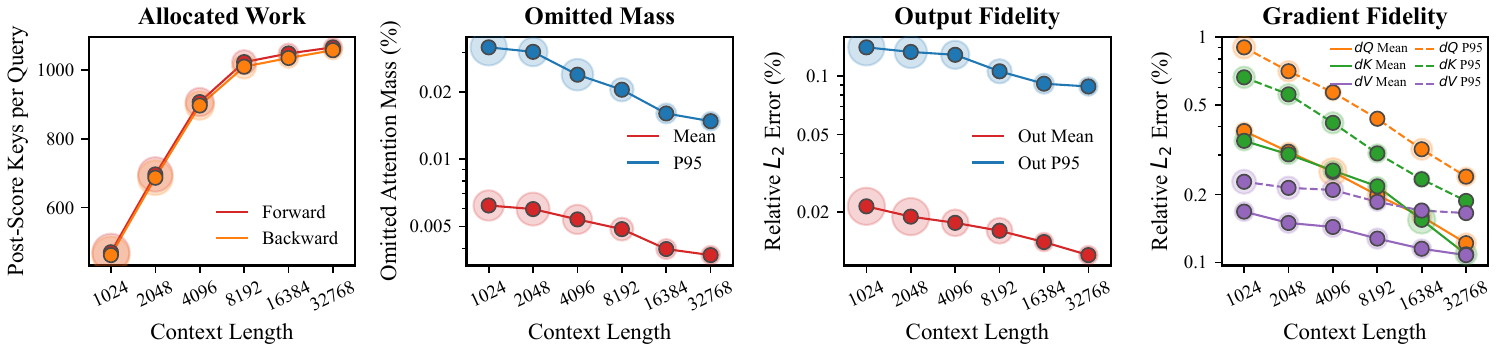}
    \vspace{-1.0em}
    \caption{
        \textbf{Operator fidelity.}
        Forward and Backward allocated work is reported as mean post-score key slots per query. Actual omitted reference probability mass, relative output error, and relative gradient errors are measured over every query position against the same operator with screening disabled. Query positions are first aggregated within each sequence-layer-head state; reported means and P95 values pool these states, and marker halos reflect confidence intervals from sequence bootstrap.
    }
    \label{fig:operator_fidelity}
    \vspace{-1.0em}
\end{figure}

\textbf{Operator Fidelity Across Context Lengths.}
We next test whether a single normalized-mass tolerance preserves the behavior of the reference attention operator as its context grows. For prefixes from 1K to 32K tokens, we run the standard screened operator and a screening-disabled reference on identical model states over the complete causal support. We measure realized post-score work, omitted reference probability mass, and relative $L_2$ errors in the output and in $\vdQ$, $\vdK$, and $\vdV$. Figure~\ref{fig:operator_fidelity} shows that, as context grows by $32\times$, mean forward work increases from 470 to 1,066 key slots per query and mean backward work from 462 to 1,058; beyond 8K, both remain between 1,010 and 1,066. This stable regime establishes 1,024 as the rounded empirical reference scale used below, induced by the canonical tolerance. Within this operator-fidelity suite and across all lengths, mean omitted probability mass is at most 0.0062\% and its P95 is at most 0.032\%; mean relative output $L_2$ error is at most 0.021\% and its P95 at most 0.14\%. Mean relative gradient errors are at most 0.38\% for $\vdQ$, 0.35\% for $\vdK$, and 0.17\% for $\vdV$; the corresponding worst P95 errors are 0.90\%, 0.66\%, and 0.23\%. The stable work and low errors support the same tolerance across forward and backward execution. Appendix~\ref{sec:appendix:operator_fidelity_setup} provides the data, sampling, gradient-probe, and aggregation protocols.

\textbf{Controlled Associative Recall.}
We use controlled associative recall~\citep{arora2024zoology} to test whether distribution-adaptive allocation preserves arbitrary long-range interactions before language priors and model-scale effects enter the evaluation. Guided by the approximately 1K scale observed above, we use raw post-score slots to control retained interactions: fixed-budget methods use a 1,024-slot ceiling after causal clipping, including DSA despite its larger value dimension, while \shortName{} continues to use the canonical tolerance and realizes approximately the same work scale through normalized-mass allocation. We vary sequence length from 1,024 to 8,192 and $d_{model}$ from 64 to 512. Figure~\ref{fig:ar} shows that \shortName{} closely tracks FullAttn as the context grows. At sequence length 8,192 and $d_{model}=512$, \shortName{} reaches 89.67\% accuracy and FullAttn reaches 89.97\%, whereas DSA and MoBA reach 52.61\% and 47.25\%, respectively, and NSA reaches 22.61\%. Thus attention-directed allocation preserves associations that fixed-support and externally selected alternatives miss under a common raw-slot reference. Appendix~\ref{sec:appendix:associative_recall_setup} provides the data construction, training protocol, and work-matching details.

\begin{figure}[!t]
    \centering
    \includegraphics[width=\textwidth]{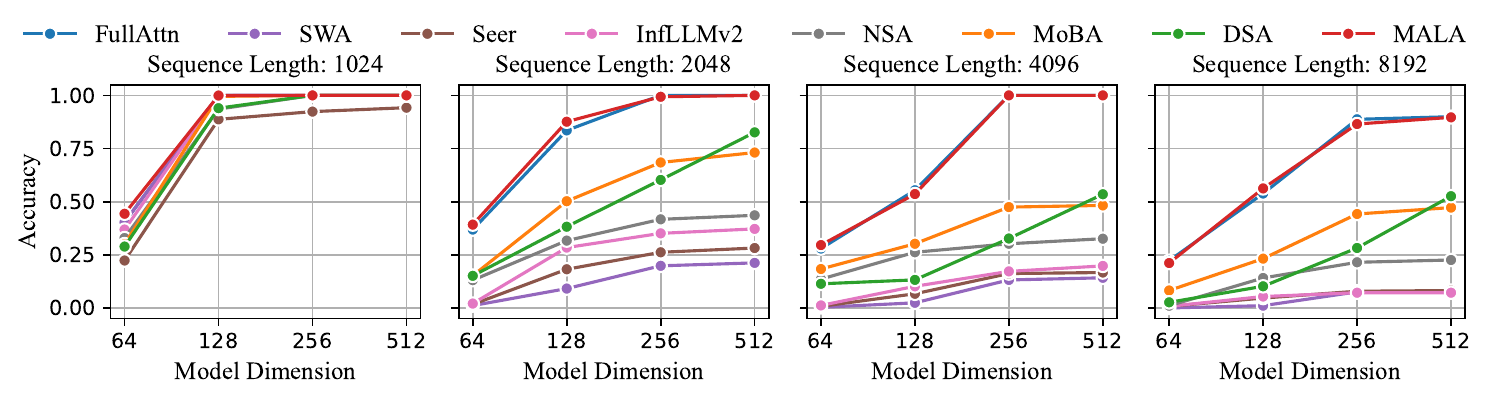}
    \vspace{-2.0em}
    \caption{
      \textbf{Associative Recall}.
      Accuracy with 256 key-value pairs across sequence lengths and model dimensions. Fixed-budget sparse methods use a common 1,024-slot ceiling, with realized raw slots determined after causal clipping; \shortName{} obtains its work from the canonical tolerance. The task directly tests the 256 bindings, with no injected distractors. \shortName{} closely tracks FullAttn as the context grows, while the other sparse mechanisms lose a substantial fraction of the associations.
    }
    \vspace{-1.0em}
    \label{fig:ar}
\end{figure}

\begin{figure}[!t]
    \centering
    \includegraphics[width=\textwidth]{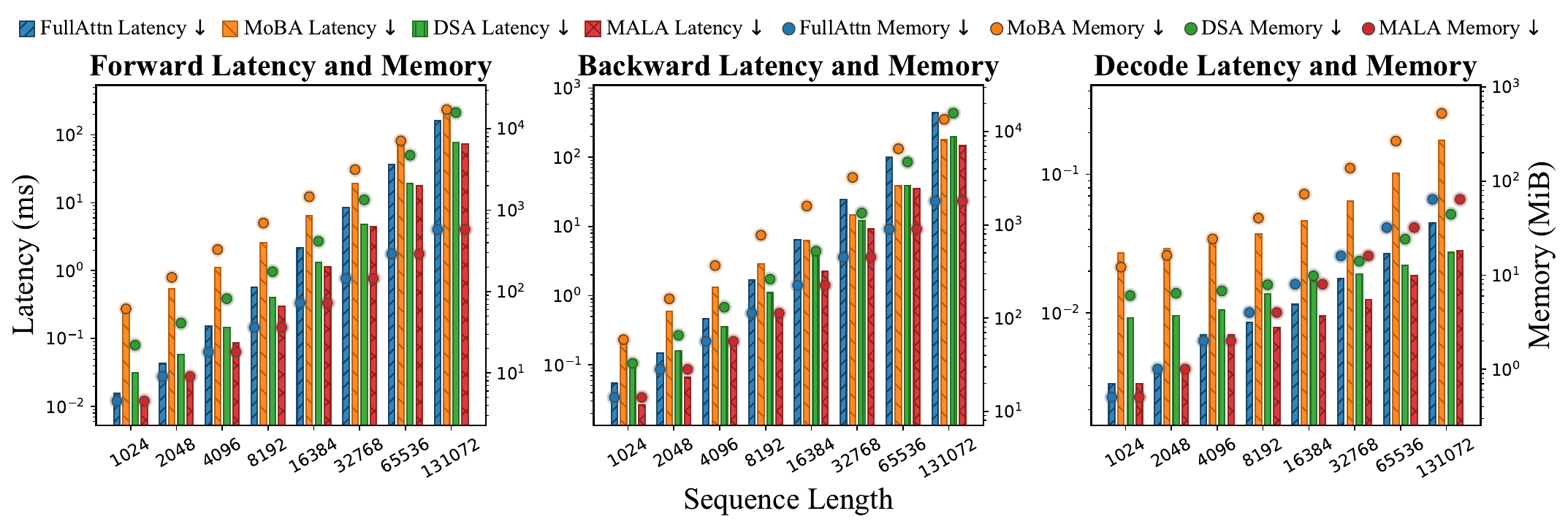}
    \vspace{-2.0em}
    \caption{
      \textbf{Operator Latency and Memory}.
      End-to-end attention-operator latency (bars, left axes) and per-rank peak memory (markers, right axes) for the training forward and backward passes and autoregressive decoding with tensor parallelism $TP=8$ on 8 H100 GPUs. Measurements include each method's complete score-discovery and selection path. \shortName{} reduces latency while retaining FullAttn-level memory through in-kernel allocation.
    }
    \label{fig:latency_memory}
    \vspace{-1.0em}
\end{figure}

\textbf{Operator Latency and Memory.}
Having established that normalized-mass allocation preserves operator behavior and long-range associations, we next measure whether it reduces the cost of complete attention operators. The sparse configurations target approximately 1,024 reference-equivalent post-score slots per query: MoBA uses 1,024 raw slots, DSA uses 256 because each slot contributes four times the reference PV work, and \shortName{} continues to use $\tau=1$ with distribution-determined work. This matching controls the work governed by the allocation decision; the benchmark~\citep{tillet2019triton} measures the complete operator, including score discovery, method-specific allocation or selection, and data movement. All measurements use the matched $TP=8$ configurations described in Appendix~\ref{sec:appendix:operator_acceleration_setup}. Figure~\ref{fig:latency_memory} shows that allocation from the attention loop translates into consistent gains as the context grows. At 128K tokens, \shortName{} reduces forward and backward latency during training by $2.2\times$ and $3.0\times$, and decoding latency during inference by $1.6\times$ relative to FullAttn. It provides the lowest forward and backward latency among the compared methods; its decoding latency is within 3\% of DSA while remaining $1.6\times$ faster than FullAttn. Across all three operators and every measured context length, its peak allocation matches FullAttn because allocation is fused into the attention loop and uses only standard attention state. These gains arise from avoiding low-contribution post-score execution; \shortName{} continues to discover scores over the complete causal context.

\begin{figure}[!t]
    \centering
    \includegraphics[width=\textwidth]{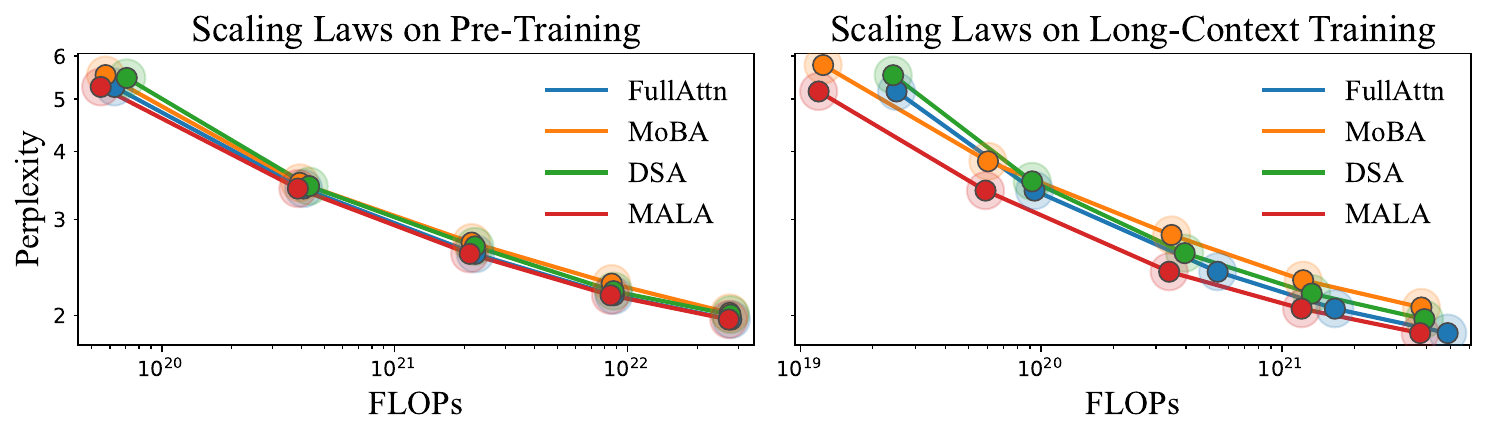}
    \vspace{-2.0em}
    \caption{
      \textbf{Scaling Laws}.
      Perplexity as a function of total training FLOPs for models from 0.6B to 14B parameters during 4K pre-training (left) and 32K long-context training (right). The FLOP accounting includes complete score discovery and method-specific allocation or selection costs. \shortName{} closely follows FullAttn in Perplexity while shifting the frontier toward lower training FLOPs.
    }
    \label{fig:scaling_laws_flops}
    \vspace{-1.0em}
\end{figure}

\begin{table}[!t]
    \centering
    \small
    \caption{
    \textbf{Model-level evaluation summary.}
    Average knowledge and reasoning scores, together with RULER scores at the native 32K context length and after YaRN extrapolation to 128K. Detailed per-task results and uncertainty estimates are reported in Appendix~\ref{sec:appendix:model_results}.
    }
    \vspace{-1.0em}
    \resizebox{\linewidth}{!}{
    \begin{tabular}{@{}llcccc@{}}
    \toprule
    \sc{Model Scale} & \sc{Method} & \sc{Knowledge} & \sc{Reasoning} & \sc{RULER 32K} & \sc{RULER 128K} \\
    \midrule
    14B & FullAttn & \underline{72.32} & \underline{64.46} & \underline{89.42} & \textbf{65.84} \\
    14B & MoBA & 70.01 & 58.07 & 86.85 & 60.17 \\
    14B & DSA & 70.11 & 62.04 & 87.59 & 63.05 \\
    14B & \textbf{\shortName{} (ours)} & \textbf{72.48} & \textbf{64.66} & \textbf{89.45} & \underline{65.75} \\
    \midrule
    32B & FullAttn & \underline{75.62} & \underline{75.67} & \underline{92.70} & \underline{82.03} \\
    32B & \textbf{\shortName{} (ours)} & \textbf{75.75} & \textbf{76.10} & \textbf{92.71} & \textbf{82.56} \\
    \bottomrule
    \end{tabular}
    }
    \vspace{-2.0em}
    \label{table:model_level_summary}
\end{table}

\textbf{Scaling Across Model Sizes.}
We next test whether normalized-mass allocation preserves language-model quality while reducing training computation across model sizes~\citep{kaplan2020scalinglawsneurallanguage,xiong2023effectivelongcontextscalingfoundation}. We study five model sizes from 0.6B to 14B under matched 4K pre-training and 32K long-context training configurations, using the same canonical tolerance for \shortName{} and the same post-score FLOP-equivalent configurations as in the operator study. Figure~\ref{fig:scaling_laws_flops} reports Perplexity against total training FLOPs, including parameterized model computation, complete score discovery, method-specific routing or indexing, and retained post-score execution. Across both stages and all model sizes, \shortName{} closely tracks FullAttn in Perplexity while shifting the quality-compute frontier toward lower FLOPs. At 14B, its Perplexity differs from FullAttn by less than 0.001 after both pre-training and long-context training, while reducing the corresponding total FLOPs by 2.5\% and 23.1\%. The larger gain at 32K follows from allocating only the post-score work warranted by the realized distribution, while the complete causal QK score-discovery cost remains included. Detailed architectures, training schedules, allocation configurations, and FLOP accounting are provided in Appendix~\ref{sec:appendix:scaling_laws_setup}.

\textbf{Model-Level Evaluation.}
We finally test whether the preceding operator fidelity and scaling trends transfer to complete language models. 
We evaluate all four 14B attention variants from the scaling-law study and extend the FullAttn-\shortName{} comparison to 32B models trained in a separate continued-training study.
Table~\ref{table:model_level_summary} shows that \shortName{} remains comparable to FullAttn in average knowledge, reasoning, native 32K retrieval, and YaRN-extrapolated 128K retrieval at both scales. At 14B, \shortName{} scores 72.48 and 64.66 on knowledge and reasoning versus 72.32 and 64.46 for FullAttn; at 32B, the corresponding averages are 75.75 and 76.10 versus 75.62 and 75.67. On native 32K RULER, \shortName{} reaches 89.45 versus 89.42 at 14B and 92.71 versus 92.70 at 32B. 
Under YaRN~\citep{peng2026yarn} extrapolation to 128K, it reaches 65.75 versus 65.84 at 14B and 82.56 versus 82.03 at 32B. Bold and underlining mark the best and second-best results, respectively, within each model-scale and context-length group. 
Detailed per-task results and uncertainty estimates are reported in Appendix~\ref{sec:appendix:model_results}. Together, these results show that attention-directed compute allocation preserves broad language-model capabilities and long-context behavior across model scales.

%% file: sections/related_work.tex
\section{Related Work}
\label{sec:related_work}

\shortName{} preserves complete causal score discovery and uses native softmax statistics to allocate post-score computation during training and inference; we relate this design to sparse attention, adaptive allocation and threshold calibration, hardware-aware sparse execution, and KV-cache management.

\textbf{Sparse attention.}
Fixed sparse attention uses local windows, structured blocks, or global tokens~\citep{child2019generating,beltagy2020longformerlongdocumenttransformer,zaheer2020big}.
Dynamic methods select tokens or blocks using retrieval scores, importance estimates, cache statistics, or learned routers~\citep{tang2024quest,lai2025flexprefill,li2024snapkv,zhang2023h2o,xiao2024infllm,qi2026pariskv,zhao2025infllmv2,gao2024seerattention,yuan2025nativesparseattentionhardwarealigned,lu2025moba,deepseekai2025deepseekv32pushingfrontieropen}.
By restricting the support before exact attention, these methods can avoid QK computation for excluded interactions.
\shortName{} retains complete causal QK score discovery and allocates post-score computation from the resulting scores and softmax state.

\textbf{Adaptive allocation and threshold calibration.}
Twilight and Double-P adapt the retained computation through hierarchical top-$p$ selection using estimated attention mass, with configurable cumulative-mass targets~\citep{lin2025twilightadaptiveattentionsparsity,ni2026doublephierarchicaltoppsparse}.
Top-Theta instead calibrates thresholds offline for individual layers, heads, and sequence lengths~\citep{berestizshevsky2026topthetaattentionsparsifyingtransformers}.
\shortName{} uses a shared normalized-contribution tolerance and derives allocation decisions directly from exact QK scores and native softmax normalizers inside the attention operator.
The same tolerance applies across heads, layers, and inputs without separately calibrated thresholds, while their realized attention distributions determine the retained post-score work.

\textbf{Hardware-aware sparse execution.}
FlashAttention reduces memory traffic through tiled execution and online softmax~\citep{dao2022flashattention,shah2024flashattention3,milakov2018onlinesoftmax}; FlexAttention and FlashInfer improve kernel programmability and serving efficiency~\citep{dong2024flexattentionprogrammingmodel,ye2025flashinferefficientcustomizableattention}.
SpargeAttention combines a predicted attention map with an online softmax-aware filter~\citep{zhang2025spargeattentionaccuratetrainingfreesparse}.
BLASST compares block maxima with running maxima and calibrates length-dependent thresholds for prefill and decoding; it also explores sparsity-aware training~\citep{yuan2025blasstdynamicblockedattention}.
\shortName{} tests scores against the accumulated or finalized softmax normalizer, using $\tau/L_q$ as a normalized-contribution tolerance.
This rule supports a forward omitted-mass guarantee and paired fused forward and backward operators, with nested backward support derived from standard attention state and a shared tolerance across training, prefill, and decoding.

\textbf{KV-cache management.}
KV-cache systems reduce inference cost by pruning, compressing, retrieving, or offloading cached states~\citep{zhang2023h2o,li2024snapkv,tang2024quest,xiao2024infllm,liu2024clusterkv,huang2025nosanativeoffloadablesparse}.
\shortName{} retains the full KV history for score discovery and value access, reducing post-score execution while leaving cache storage and placement as separate optimization opportunities.

%% file: sections/conclusion.tex
\section{Conclusion}

We introduced \fullName{} (\shortName{}), a fused attention primitive that allocates post-score computation from normalized contributions under a shared tolerance.
\shortName{} preserves complete causal QK score discovery and pairs online forward allocation with backward allocation from the saved normalizer, deriving nested retained support using standard attention state.
The same tolerance governs training, and inference.
QK complexity remains quadratic, with savings arising from skipped post-score execution.

The matched-work allocation study shows that \shortName{} closely approaches the per-instance reference-mass oracle under the same total post-score work.
Operator-fidelity evaluations show low output and gradient errors across context lengths with a shared tolerance, while controlled associative recall shows that \shortName{} closely tracks FullAttn as context grows.
The tensor-parallel operator benchmarks show lower forward, backward and decoding attention latency while retaining FullAttn-level peak operator memory.
The scaling-law experiments from 0.6B to 14B parameters show perplexity comparable to FullAttn with lower training FLOPs, and the model-level evaluations at 14B and 32B show comparable knowledge, reasoning, and long-context retrieval scores.
These results indicate that allocating post-score computation according to normalized attention contributions can retain the evaluated capabilities of FullAttn while reducing attention computation.

%% file: sections/appendix.tex
\newpage
\section{Complete Forward and Backward Algorithms}
\label{sec:appendix:algorithms}

Algorithms~\ref{alg:fsa_forward} and~\ref{alg:fsa_backward} give the complete tiled procedures corresponding to the online and offline allocation rules in Section~\ref{sec:fsa_rule}.

\begin{algorithm}[H]
    \small
    \caption{\fullName{} Forward}
    \label{alg:fsa_forward}
    \vspace{-0.25em}
    \begin{algorithmic}[1]
        \REQUIRE Matrices $\vQ, \vO \in \mathbb{R}^{N_q \times d_h}$ and $\vK, \vV \in \mathbb{R}^{N_k \times d_h}$. Vector $\vell \in \mathbb{R}^{N_q}$. Thresh $\tau$.
        \STATE Divide $\vQ, \vO$ into $T_r$ blocks of size $B_q$ and $\vK, \vV$ into $T_c$ blocks of size $B_k$.
        \FOR{$1 \le i \le T_r$}
            \STATE Initialize $\vO_i = (0) \in \mathbb{R}^{B_q \times d_h}$, $\vell_i = (0) \in \mathbb{R}^{B_q}$, and $\vm_i = (-\infty) \in \mathbb{R}^{B_q}$.
            \STATE Compute the per-row visible-key counts $\vL_i$.
            \STATE Load $\vQ_i$.
            \FOR{$T_c \ge j \ge 1$}
                \STATE Load $\vK_j$.
                \STATE Compute $\vS_{ij} = \vQ_i \vK_j^T \in \mathbb{R}^{B_q \times B_k}$ and Apply causal mask to $\vS_{ij}$ at boundaries.
                \STATE Compute $\tilde{m}_{ij} = \mathrm{rowmax}(\vS_{ij}) \in \mathbb{R}^{B_q}$.
                \IF{$\tilde{m}_{ij} - \vm_i - \log(\vell_i) < \log(\tau / \vL_i)$ for all rows}
                    \STATE \textbf{skip} block $(i, j)$ and \textbf{continue}.
                \ENDIF
                \STATE Load $\vV_j$.
                \STATE Compute $\vm_i^{\mathrm{new}} = \max(\vm_i, \tilde{m}_{ij})$.
                \STATE Compute $\tilde{\vP}_{ij} = \exp(\vS_{ij} - \vm_i^{\mathrm{new}}) \in \mathbb{R}^{B_q \times B_k}$.
                \STATE Compute $\tilde{\vell}_{ij} = \mathrm{rowsum}(\tilde{\vP}_{ij}) \in \mathbb{R}^{B_q}$ and $\vell_i^{\mathrm{new}} = \exp({\vm_i - \vm_i^{\mathrm{new}}}) \vell_i + \tilde{\vell}_{ij} \in \mathbb{R}^{B_q}$.
                \STATE Update $\vO_i\leftarrow \exp({\vm_i - \vm_i^{\mathrm{new}}})\vO_i + \tilde{\vP}_{ij} \vV_j \in \mathbb{R}^{B_q \times d_h}$.
                \STATE Update $\vell_i \leftarrow \vell_i^{\mathrm{new}}$ and $\vm_i \leftarrow \vm_i^{\mathrm{new}}$.
            \ENDFOR
            \STATE Update $\vO_i \leftarrow \diag(\vell_i)^{-1} \vO_i$ and $\vell_i \leftarrow \vm_i + \log(\vell_i)$.
            \STATE Store $\vO_i, \vell_i$.
        \ENDFOR
        \STATE Return $\vO, \vell$.
    \end{algorithmic}
    \vspace{-0.25em}
\end{algorithm}

\begin{algorithm}[H]
    \small
    \caption{\fullName{} Backward}
    \label{alg:fsa_backward}
    \vspace{-0.25em}
    \begin{algorithmic}[1]
        \REQUIRE Matrices $\vQ, \vO, \vdQ, \vdO \in \mathbb{R}^{N_q \times d_h}$ and $\vK, \vV, \vdK, \vdV \in \mathbb{R}^{N_k \times d_h}$. Vector $\vell \in \mathbb{R}^{N_q}$. Thresh $\tau$.
        \STATE Divide $\vQ, \vO, \vdQ, \vdO$ into $T_r$ blocks of size $B_q$ and $\vK, \vV, \vdK, \vdV$ into $T_c$ blocks of size $B_k$.
        \FOR{$1 \le j \le T_c$}
            \STATE Initialize $\vdK_j = (0) \in \mathbb{R}^{B_k \times d_h}, \vdV_j = (0) \in \mathbb{R}^{B_k \times d_h}$.
            \STATE Load $\vK_j, \vV_j$.
            \FOR{$j \le i \le T_r$}
                \STATE Compute the per-row visible-key counts $\vL_i$.
                \STATE Load $\vQ_i$.
                \STATE Compute $\vS_{ji} = \vK_j \vQ_i^T \in \mathbb{R}^{B_k \times B_q}$ and Apply causal mask to $\vS_{ji}$ at boundaries.
                \IF{$\vS_{ji} - \vell_i^\top < \log(\tau / \vL_i)$ for all rows}
                    \STATE \textbf{skip} block $(j, i)$ and \textbf{continue}.
                \ENDIF
                \STATE Load $\vO_i, \vdO_i$.
                \STATE Compute $\vP_{ji} = \exp(\vS_{ji} - \vell_i^\top) \in \mathbb{R}^{B_k \times B_q}$ and $\vdP_{ji} = \vV_j \vdO_i^\top \in \mathbb{R}^{B_k \times B_q}$.
                \STATE Compute $\vdS_{ji} = \vP_{ji} \circ (\vdP_{ji} - \mathrm{rowsum}(\vO_i \circ \vdO_i)^\top) \in \mathbb{R}^{B_k \times B_q}$.
                \STATE Update $\vdV_j \leftarrow \vdV_j + \vP_{ji} \vdO_i$ and $\vdK_j \leftarrow \vdK_j + \vdS_{ji} \vQ_i$.
                \STATE Store $\vdQ_i \leftarrow \vdQ_i + \vdS_{ji}^\top \vK_j$.
            \ENDFOR
            \STATE Store $\vdK_j, \vdV_j$.
        \ENDFOR
        \STATE Return $\vdQ, \vdK, \vdV$.
    \end{algorithmic}
    \vspace{-0.25em}
\end{algorithm}

\subsection{Fused Execution Details}
\label{sec:appendix:fused_execution}

\textbf{Fused forward execution.}
The \shortName{} decision is embedded in the ordinary tiled attention loop. Algorithm~\ref{alg:fsa_forward} parallelizes over query blocks. Within each query block, legal key blocks are visited in descending order, beginning with the diagonal. For each legal key block, the kernel first loads K and computes QK. It loads V and executes the online-softmax and PV path only when the contribution test retains the tile. The running maximum, normalization sum, and output accumulator remain on chip exactly as in FlashAttention. After the traversal, the operator writes the output $\vO$ and final $\vell$ required by backward. Allocation remains implicit in these standard attention states. The same fused forward operator serves training and inference prefill, with query blocks executing independently in both settings.

\textbf{Fused backward execution.}
Algorithm~\ref{alg:fsa_backward} uses the standard key-major traversal so that each program can accumulate $\vdK_j$ and $\vdV_j$ locally while contributing to $\vdQ_i$. For each key block, query blocks are visited in ascending order, beginning with the diagonal. QK is recomputed before the skip decision. For a retained tile, the kernel reconstructs its probability from the saved $\vell$ and executes the usual attention-backward equations; for a skipped tile, it saves the corresponding output-gradient load and value-dependent path. The log-sum-exp state $\vell$ saved by the forward pass also serves as the allocation state for backward, yielding a fused attention-backward primitive.

\textbf{Fused decoding execution.}
Autoregressive decoding applies the same forward allocation rule with a short query and a longer KV sequence. The kernel evaluates every causal QK tile, then loads V and executes softmax and PV for retained tiles. All cached keys therefore remain score-accessible. Under split-KV decoding, each split initializes its own running maximum and normalization sum and tests tiles against a partial normalizer. This typically retains at least the work selected under the corresponding unsplit normalizer. The partial outputs are then combined by the standard online-softmax reduction. The persistent KV cache remains intact, and the decode-memory claim concerns the attention operator's working set.

\textbf{Training and inference consistency.}
The same probability tolerance governs forward, backward, prefill, and decoding. Attention distributions and available normalizers determine the realized work in each setting. Backward can be more selective than forward because it tests against the final normalizer; this directly yields the support-nesting behavior described in Section~\ref{sec:fsa_rule}.

\section{Forward Approximation Guarantees}
\label{sec:appendix:fsa_mass_bound}

We state the forward guarantee for a single query row; applying it independently to every valid row gives the result used in Section~\ref{sec:fsa_rule}. Let $s_k$ be the score of causally visible key $k$, and let $R$ and $O$ denote the sets of keys retained and omitted by \shortName{}, respectively. Define the final retained and omitted unnormalized masses as
\begin{equation}
    Z_R=\sum_{k\in R}\exp(s_k),
    \qquad
    Z_O=\sum_{k\in O}\exp(s_k).
    \label{eq:appendix:retained_omitted_mass}
\end{equation}

\begin{proposition}[Omitted probability mass]
For a query with $L_q$ causally visible keys, the forward rule in Equation~\ref{eq:fsa_online_test} bounds the total probability mass assigned by dense softmax to omitted keys as
\begin{equation}
    \epsilon_q
    \defeq
    \frac{Z_O}{Z_R+Z_O}
    \le
    \frac{\tau}{1+\tau}.
    \label{eq:appendix:omitted_mass_bound}
\end{equation}
\end{proposition}

\begin{proof}
Consider an omitted key $k$ and the retained normalizer $Z_t$ accumulated when its tile is tested. The first legal tile is necessarily retained, so $Z_t>0$ whenever a skip occurs. Equation~\ref{eq:fsa_online_test} and the tile row maximum imply
\begin{equation}
    \frac{\exp(s_k)}{Z_t}<\frac{\tau}{L_q}.
\end{equation}
The final retained normalizer satisfies $Z_R\ge Z_t$, hence $\exp(s_k)/Z_R<\tau/L_q$. Summing over omitted keys and using $|O|\le L_q$ gives $Z_O/Z_R\le\tau$. Therefore
\begin{equation}
    \epsilon_q
    =\frac{Z_O/Z_R}{1+Z_O/Z_R}
    \le\frac{\tau}{1+\tau}.
\end{equation}
\end{proof}

Let $p$ be the dense softmax distribution and let $\widetilde{p}$ be softmax renormalized over $R$, with zero probability on $O$. The retained probabilities gain total mass $\epsilon_q$, while the omitted probabilities lose the same mass, so
\begin{equation}
    \lVert p-\widetilde{p}\rVert_1
    =2\epsilon_q
    \le\frac{2\tau}{1+\tau}.
    \label{eq:appendix:distribution_l1_bound}
\end{equation}
If $\lVert\vv_k\rVert\le V_{max}$ for every visible value, the dense output $\vo=\sum_k p_k\vv_k$ and the retained output $\widetilde{\vo}=\sum_k\widetilde{p}_k\vv_k$ consequently satisfy
\begin{equation}
    \lVert\vo-\widetilde{\vo}\rVert
    \le
    \sum_k |p_k-\widetilde{p}_k|\lVert\vv_k\rVert
    \le
    \frac{2\tau}{1+\tau}V_{\max}.
    \label{eq:appendix:output_bound}
\end{equation}
These statements assume exact arithmetic. In the fused implementation, the running normalizer and contribution test are evaluated in finite precision, so the formal bounds hold up to the corresponding rounding error. Section~\ref{sec:experiments} measures omitted mass and output and gradient fidelity for the implemented operators directly.

\section{Experiment Settings}
\label{sec:appendix:experiment_setup}

Within each experiment, attention variants use matched model scales, depth, hidden size, data, and optimization settings while retaining their native attention-specific configurations. Unless stated otherwise, \shortName{} uses the canonical tolerance $\tau=1$, with realized work determined separately for each layer, head, input, and sequence length. The matched-work and operator-fidelity studies use the same 14B \shortName{} checkpoint obtained after the 32K long-context stage of the scaling-law study. Our software stack uses NVIDIA PyTorch container images~\citep{nv2022pytorch} and the Transformers framework~\citep{wolf-etal-2020-transformers}. Scaling-law ~\citep{kaplan2020scalinglawsneurallanguage,xiong2023effectivelongcontextscalingfoundation}and continued-training experiments use Megatron-LM~\citep{megatron-lm} on 128 NVIDIA H100 GPUs. Benchmark evaluation uses the EleutherAI LM Evaluation Harness~\citep{eval-harness} on 8 NVIDIA H100 GPUs. Operator benchmarking uses 8 H100 GPUs with tensor parallelism $TP=8$.

\subsection{Work and Cost Accounting}
\label{sec:appendix:work_accounting}

We distinguish three quantities throughout the experiments: realized allocation, architecture-normalized allocation, and complete system cost.

\paragraph{Mean post-score key slots per query.}
For method $a$, let $s_{a,q}$ be the number of key slots for which query row $q$ executes the post-score path after QK score discovery. A slot is counted when it participates in an executed softmax/value/PV tile or sparse-attention kernel call. Consequently, the count follows the operator's actual tile or block granularity and includes aligned or masked slots when the kernel still executes their post-score arithmetic; positions removed before this path are not counted. We report
\begin{equation}
    \bar{s}_a
    =
    \frac{1}{|\mathcal{Q}|}
    \sum_{q\in\mathcal{Q}} s_{a,q},
    \label{eq:appendix:mean_post_score_slots}
\end{equation}
where $\mathcal{Q}$ contains all query-row instances included by the experiment, including the relevant examples, layers, heads, and positions. This realized mean differs from a nominal Top-$K$ or block budget. In particular, causal clipping, block expansion, and tile sharing can change $\bar{s}_a$ even when two methods receive the same configured integer budget. We use this raw quantity for within-architecture comparisons, including the matched-work allocation study, to describe how much work \shortName{} realizes at a fixed tolerance, and in associative recall to define a common retained-interaction reference across methods.

\paragraph{Post-score FLOP-equivalent work.}
Raw slot counts do not define a compute-matched comparison when attention variants use different value dimensions or numbers of QO heads. For forward attention, the dominant allocated matrix multiplication is PV, whose cost is $2h_a d_{v,a}\bar{s}_a$ FLOPs per query token, summed over all QO heads, when one multiply-accumulate counts as two FLOPs. Relative to a reference configuration with $h_{\mathrm{ref}}$ QO heads and value dimension $d_{v,\mathrm{ref}}$, we therefore define
\begin{equation}
    \bar{s}^{\mathrm{eq}}_a
    =
    \bar{s}_a
    \frac{h_a d_{v,a}}
         {h_{\mathrm{ref}}d_{v,\mathrm{ref}}}.
    \label{eq:appendix:post_score_equivalent_work}
\end{equation}
We use the matched FullAttn/\shortName{} head configuration as the reference within each experiment. Thus, when the number of QO heads is shared, one DSA slot with $d_v=512$ is approximately four 128-dimensional reference slots. Softmax and elementwise operations are lower-order terms in this matching proxy and remain included in measured operator latency. We use this normalized quantity when a cross-architecture comparison is intended to control allocated computation, specifically in the operator and scaling-law configurations. Associative recall instead matches raw mean post-score slots, because it controls the number of retained interactions rather than their architecture-dependent FLOPs.

\paragraph{Full operator cost.}
Post-score equivalence controls only the computation governed by the allocation decision. It excludes score discovery and method-specific selection. Full operator cost additionally includes \shortName{}'s complete causal QK score discovery, MoBA's pooling, routing, Top-$K$, index construction, rearrangement, and merging, and DSA's indexer preparation, quantization, dense index scores, Top-$K$, indices, and sparse MLA score computation. We use synchronized end-to-end operator latency and peak memory for the systems comparison. For scaling-law plots, we count these sequence-dependent terms together with the shared model-training term. This separation prevents a raw key count from hiding architectural differences and prevents post-score matching from being presented as an end-to-end compute match.

\subsection{Normalized-Mass Allocation at Matched Work Setup}
\label{sec:appendix:matched_work_setup}

Throughout the matched-work and operator-fidelity studies, omitted mass denotes $1-\exp(\ell_{\mathrm{screened}}-\ell_{\mathrm{reference}})$, the probability mass assigned by the dense reference softmax to positions omitted by the screened operator. We report this fraction as a percentage in the experimental results.

We evaluate complete causal attention on 256 language-modeling sequences of length 8,192, covering every query position in every layer and query head. The resulting evaluation contains 58,720,256 runtime allocation decisions. For each decision, a screening-disabled reference provides the probability mass of every candidate region. The diagnostic controls retain the highest-mass regions under this reference; \shortName{} uses its actual online allocation and selection rule.

All policies execute exactly the same total number of post-score key slots, with a shared mean of approximately 1,024 slots per query induced by \shortName{}'s realized allocation. All three diagnostic controls use the finalized screening-disabled distribution to rank candidate regions by probability mass at the same granularity as \shortName{}; they differ only in the number of slots assigned to each decision. The position-only control averages \shortName{}'s realized work across inputs, layers, and heads at each query position, producing one shared position-dependent allocation. The static layer-head-position control averages only across inputs, producing a separate allocation for each query position, layer, and head that remains fixed across sequences. The per-instance control uses the exact number of slots realized by \shortName{} for each input, query position, layer, and head. Given an allocation, each control retains the highest-mass candidate regions. Data-independent integer rounding makes the two static allocations match \shortName{}'s total slot count exactly. The per-instance control therefore separates the value of instance-dependent work allocation from the approximation introduced by making the selection decision online. Because every row in this study uses the same attention representation and post-score kernel shape, equality in total slots is also equality in post-score FLOPs.

\subsection{Operator Fidelity Across Context Lengths Setup}
\label{sec:appendix:operator_fidelity_setup}

We randomly sample 256 sequences from a mixture of knowledge, reasoning, and retrieval data and evaluate prefixes of 1,024, 2,048, 4,096, 8,192, 16,384, and 32,768 tokens. At every layer and query position, we apply the standard screened operator and a screening-disabled reference to the same queries, keys, values, causal support, and kernel configuration. We record the realized forward and backward post-score key slots per query and recover the probability mass omitted by \shortName{} from the two softmax normalizers.

We measure relative $L_2$ errors in the output and in $\vdQ$, $\vdK$, and $\vdV$. Each relative error is the $L_2$ norm of the difference divided by the $L_2$ norm of the corresponding reference tensor, and is reported as a percentage. Backward measurements use a deterministic Rademacher probe for $\vdO$, providing a reproducible direction through the complete attention backward operator. We aggregate complete operators at head level, yielding 409,600 query-head-layer measurements and 81,920 KV-head-layer measurements at each context length. Means and percentiles are computed over these measurements; confidence intervals use sequence-level bootstrap so that positions, layers, and heads from the same sequence remain clustered.

\subsection{Associative Recall Setup}
\label{sec:appendix:associative_recall_setup}

Following prior work~\citep{arora2024zoology}, we evaluate associative recall with 256 key-value pairs. We consider sequence lengths of 1,024, 2,048, 4,096, and 8,192 and model dimensions $d_{model}\in\{64,128,256,512\}$. The repeated key-value content is placed contiguously at the beginning of each sequence, occupying the first $2n_{\mathrm{kv}}n_{\mathrm{pass}}$ tokens, where $n_{\mathrm{kv}}=256$ and $n_{\mathrm{pass}}$ is the number of passes. We then sample $n_{\mathrm{kv}}$ query positions without replacement from the remaining positions according to $p(t)\propto t^{a-1}$, where $t$ is the position within this suffix. We use the default $a=0.01$, so queries occur more frequently near the beginning of the suffix and become progressively sparser toward its end. We do not insert random padding or additional noise tokens. This setup tests recall without adding a separate task of filtering random padding or noise. The dataset contains 250K training examples and 1K test examples, and all models are trained for 100 epochs.

FullAttn attends to the complete causal history. Fixed-budget sparse variants use a common nominal ceiling of 1,024 post-score key slots per query. SWA, Seer, InfLLMv2, NSA, MoBA, and DSA retain their native support or selection mechanisms; their realized means are computed after causal clipping and implementation-specific block expansion. \shortName{} uses the canonical tolerance $\tau=1$; the fused operator realizes $\bar{s}$ near the same empirical work scale. This comparison uses raw retained interactions as its reference, aligning the number of accessible associations across methods. DSA follows the same raw-slot ceiling even though its $d_v=512$ representation makes each retained slot more expensive than a 128-dimensional \shortName{} or MoBA slot.

\begin{table}[H]
    \centering
    \small
    \caption{
    \textbf{Associative-recall allocation configurations.}
    Fixed-budget sparse methods use a common 1,024-slot ceiling, with realized work computed after causal clipping and native block expansion. The raw-slot comparison aligns retained interactions; DSA therefore uses the same slot ceiling despite its larger value dimension. \shortName{} uses $\tau=1$, with work determined by normalized-mass allocation.
    }
    \label{tab:appendix:ar_work_config}
    \resizebox{\linewidth}{!}{
    \begin{tabular}{@{}lcccc@{}}
    \toprule
    \sc{Method} & \sc{Allocation Rule} & $d_v/d_{v,\mathrm{ref}}$ & \sc{Nominal Slot Ceiling} & \sc{Raw-Slot Reference} \\
    \midrule
    FullAttn & Complete causal history & $1\times$ & Full & Full \\
    SWA / Seer / InfLLMv2 / NSA & Native support or selector & $1\times$ & 1,024 & after causal clipping \\
    MoBA & Native block routing & $1\times$ & 1,024 & after causal clipping \\
    DSA & Token-level Top-$K$ & $4\times$ & 1,024 & after causal clipping \\
    \textbf{\shortName{}} & Normalized-mass tolerance $\tau=1$ & $1\times$ & - & measured near reference \\
    \bottomrule
    \end{tabular}
    }
\end{table}

\subsection{Operator Latency and Memory Setup}
\label{sec:appendix:operator_acceleration_setup}

We benchmark causal attention at sequence lengths from 1,024 to 131,072 on 8 NVIDIA H100 GPUs with tensor parallelism $TP=8$. All methods use batch size one. Unless a configuration runs out of memory, latency is averaged over 100 iterations after 25 warm-up iterations. Latency is the synchronized wall-clock time of the tensor-parallel group, and peak memory is the maximum per-rank allocation across the eight ranks. Table~\ref{tab:appendix:operator_configs} gives the operator shapes and allocation rules. Sparse configurations target approximately 1,024 reference-equivalent forward post-score slots per query; \shortName{} realizes its work from $\tau=1$, while MoBA and DSA use their native block and token selectors. FullAttn processes the complete causal context.

The measurements cover the complete attention-operator path. MoBA includes representative construction, router logits, Top-$K$, index construction, rearrangement, sparse attention, and output merging. DSA includes BF16 indexer inputs, FP8 quantization, dense indexer logits, Top-$K$/indices, and sparse MLA. \shortName{} includes complete causal QK score discovery, its fused contribution test, and the retained post-score path. Model-level QKV/output projections, MLP layers, optimizer updates, and pipeline communication are outside this module benchmark. Peak memory is measured over the same operator path and represents operator working memory; persistent serving-time KV-cache capacity remains unchanged. Because \shortName{} allocation uses only standard attention state and retains all KV entries, its expected memory behavior is parity with the matched FullAttn operator.

\footnotetext[1]{The FullAttn implementation is available at \href{https://github.com/Dao-AILab/flash-attention}{\nolinkurl{github.com/Dao-AILab/flash-attention}}.}
\footnotetext[2]{The MoBA implementation is available at \href{https://github.com/MoonshotAI/MoBA}{\nolinkurl{github.com/MoonshotAI/MoBA}}.}
\footnotetext[3]{The DSA indexer implementation is available at \href{https://github.com/deepseek-ai/DeepGEMM}{\nolinkurl{github.com/deepseek-ai/DeepGEMM}}.}
\footnotetext[4]{The DSA implementation is available at \href{https://github.com/Dao-AILab/flash-attention/tree/main/flash_attn/cute}{\nolinkurl{github.com/Dao-AILab/flash-attention/flash_attn/cute}}.}

\begin{table}[H]
    \centering
    \small
    \caption{
    \textbf{Operator benchmark configurations.}
    The sparse methods target comparable forward post-score FLOP-equivalent work, while reported latency and memory include their complete score-discovery and selection paths.
    }
    \label{tab:appendix:operator_configs}
    \resizebox{\linewidth}{!}{
    \begin{tabular}{@{}lccccccc@{}}
    \toprule
    \sc{Method} & $n_h$ & $n_{h_{kv}}$ & $d_{qk}$ & $d_v$ & \sc{Allocation} & \sc{Target Eq. Slots} & \sc{Precision} \\
    \midrule
    FullAttn\footnotemark[1] & 64 & 8 & 128 & 128 & Full causal context & Full & BF16 \\
    MoBA\footnotemark[2] & 64 & 64 & 128 & 128 & $B=128$, native routing & $\approx$1,024 & BF16 \\
    DSA\footnotemark[3]\textsuperscript{,}\footnotemark[4] & 64 & 1 & 576 & 512 & Top-$K=256$, FP8 indexer & $\approx$1,024 & BF16/FP8 \\
    \textbf{\shortName{}} & 64 & 8 & 128 & 128 & $\tau=1$, fused allocation & measured $\approx$1,024 & BF16 \\
    \bottomrule
    \end{tabular}
    }
\end{table}

\subsection{Scaling Laws Setup}
\label{sec:appendix:scaling_laws_setup}

The scaling-law~\citep{kaplan2020scalinglawsneurallanguage,xiong2023effectivelongcontextscalingfoundation} experiments use SmolLMCorpus~\citep{benallal2024smollmcorpus}, the Qwen3 tokenizer~\citep{qwen32025}, the AdamW optimizer~\citep{Loshchilov2017FixingWD}, the WSD learning-rate schedule~\citep{hägele2024scalinglawscomputeoptimaltraining}, and the compute-optimal scaling principles~\citep{li2025predictablescalei,hoffmann2022empirical}. We use a fixed random seed of 42 for all scaling-law runs.
We summarize the meaning of the columns in Table~\ref{tab:scaling_laws_configurations} and clarify which hyperparameters are used by each attention variant.

\begin{itemize}
    \item \textbf{Params}: target model scale. Parameter counts are matched approximately because the attention projections differ across variants.
    \item \textbf{PT/LCT Tok}: total numbers of tokens used in pre-training (PT) and long-context training (LCT), respectively. Each entry reports the two values in the order \texttt{PT, LCT}.
    \item \textbf{Batch}: tokens per optimization step.
    \item \textbf{PT/LCT LR}: learning rates used during PT and LCT, respectively. Each entry reports the two values in the order \texttt{PT, LCT}.
    \item $n_{layers}$: number of Transformer layers.
    \item $d_{model}$: model hidden size.
    \item $n_h$: number of QO heads.
    \item $n_{h_{kv}}$: number of KV heads.
    \item $d_{qk}$ and $d_v$: per-head dimensions of the query/key and value representations, respectively.
    \item \textbf{PS FLOPs/Slot}: forward post-score PV FLOPs contributed by one raw key slot, aggregated over all QO heads. Counting one multiply-accumulate as two FLOPs, this is $2h d_v$. It is the per-slot factor used by Equation~\ref{eq:appendix:post_score_equivalent_work}, not the full attention cost.
    \item \textbf{PT Alloc} and \textbf{LCT Alloc}: allocation controls used during PT and LCT. FullAttn entries give the complete context length, MoBA and DSA entries give nominal raw selected-key budgets, and \shortName{} entries give its normalized-mass tolerance.
\end{itemize}

For models from 0.6B to 14B, training proceeds in two stages. PT uses a sequence length of 4,096 and a WSD learning-rate schedule~\citep{hägele2024scalinglawscomputeoptimaltraining} whose peak learning rate is reported in the PT entry of Table~\ref{tab:scaling_laws_configurations}. The decay phase ends at the corresponding LCT learning rate, which is $0.1\times$ the PT peak. LCT initializes from the final PT checkpoint, increases sequence length to 32,768, and uses the reported constant LCT learning rate.

\begin{itemize}
    \item \textbf{FullAttn}: standard causal scaled dot-product attention over the complete 4,096-token PT or 32,768-token LCT sequence.
    \item \textbf{MoBA}: block-routed attention with one KV head per QO head and a nominal 1,024-slot budget during both PT and LCT.
    \item \textbf{DSA}: sparse latent attention with one KV head, $d_{qk}=576$, and $d_v=512$. Its post-score PV cost per raw slot is four times that of MoBA or \shortName{}, so we use a nominal budget of 256 during both PT and LCT.
    \item \textbf{\shortName{}}: grouped-query attention with eight KV heads and the canonical tolerance $\tau=1$ throughout PT and LCT. Normalized-mass allocation yields a realized mean near the 1,024-slot reference level.
\end{itemize}

Each method retains its native head configuration: FullAttn and \shortName{} use eight KV heads, MoBA uses one KV head per QO head, and DSA uses a single latent KV head. The sparse variants approximately match forward post-score FLOP-equivalent work across these configurations. Full training FLOPs additionally include score discovery, routing, indexing, and Top-$K$ costs as applicable. FullAttn attends to the complete context and serves as the dense reference.

\begin{table}[H]
    \centering
    \small
    \caption{
    \textbf{Self-Attention Variants Scaling Laws Configurations.}
    Complete model, training, and attention configurations used in the scaling-law experiments. PS FLOPs/Slot describes allocated forward post-score work; the FLOP-axis additionally includes complete score discovery and method-specific selection.
    }
    \label{tab:scaling_laws_configurations}
    \resizebox{\linewidth}{!}{
    \begin{tabular}{@{}lccccccccccccc@{}}
    \toprule
    \sc{Algos} & \sc{Params} & \sc{PT/LCT Tok} & \sc{Batch} & \sc{PT/LCT LR} & $n_{layers}$ & $d_{model}$ & $n_h$ & $n_{h_{kv}}$ & $d_{qk}$ & $d_v$ & \sc{PS FLOPs/Slot} & \sc{PT Alloc} & \sc{LCT Alloc} \\
    \midrule
    FullAttn & $\approx$0.6B & 12B,1.5B & 0.256M & 1e-3,1e-4 & 28 & 1024 & 16 & 8 & 128 & 128 & 4096 & 4096 & 32768 \\
    MoBA & $\approx$0.6B & 12B,1.5B & 0.256M & 1e-3,1e-4 & 28 & 1024 & 16 & 16 & 128 & 128 & 4096 & $K=1024$ & $K=1024$ \\
    DSA & $\approx$0.6B & 12B,1.5B & 0.256M & 1e-3,1e-4 & 28 & 1024 & 16 & 1 & 576 & 512 & 16384 & $K=256$ & $K=256$ \\
    \shortName{} & $\approx$0.6B & 12B,1.5B & 0.256M & 1e-3,1e-4 & 28 & 1024 & 16 & 8 & 128 & 128 & 4096 & $\tau=1$ & $\tau=1$ \\
    \midrule
    FullAttn & $\approx$1.7B & 34B,4B & 0.512M & 8e-4,8e-5 & 28 & 2048 & 16 & 8 & 128 & 128 & 4096 & 4096 & 32768 \\
    MoBA & $\approx$1.7B & 34B,4B & 0.512M & 8e-4,8e-5 & 28 & 2048 & 16 & 16 & 128 & 128 & 4096 & $K=1024$ & $K=1024$ \\
    DSA & $\approx$1.7B & 34B,4B & 0.512M & 8e-4,8e-5 & 28 & 2048 & 16 & 1 & 576 & 512 & 16384 & $K=256$ & $K=256$ \\
    \shortName{} & $\approx$1.7B & 34B,4B & 0.512M & 8e-4,8e-5 & 28 & 2048 & 16 & 8 & 128 & 128 & 4096 & $\tau=1$ & $\tau=1$ \\
    \midrule
    FullAttn & $\approx$4B & 80B,10B & 1M & 6e-4,6e-5 & 32 & 2560 & 32 & 8 & 128 & 128 & 8192 & 4096 & 32768 \\
    MoBA & $\approx$4B & 80B,10B & 1M & 6e-4,6e-5 & 32 & 2560 & 32 & 32 & 128 & 128 & 8192 & $K=1024$ & $K=1024$ \\
    DSA & $\approx$4B & 80B,10B & 1M & 6e-4,6e-5 & 32 & 2560 & 32 & 1 & 576 & 512 & 32768 & $K=256$ & $K=256$ \\
    \shortName{} & $\approx$4B & 80B,10B & 1M & 6e-4,6e-5 & 32 & 2560 & 32 & 8 & 128 & 128 & 8192 & $\tau=1$ & $\tau=1$ \\
    \midrule
    FullAttn & $\approx$8B & 164B,20B & 1.6M & 4e-4,4e-5 & 36 & 4096 & 32 & 8 & 128 & 128 & 8192 & 4096 & 32768 \\
    MoBA & $\approx$8B & 164B,20B & 1.6M & 4e-4,4e-5 & 36 & 4096 & 32 & 32 & 128 & 128 & 8192 & $K=1024$ & $K=1024$ \\
    DSA & $\approx$8B & 164B,20B & 1.6M & 4e-4,4e-5 & 36 & 4096 & 32 & 1 & 576 & 512 & 32768 & $K=256$ & $K=256$ \\
    \shortName{} & $\approx$8B & 164B,20B & 1.6M & 4e-4,4e-5 & 36 & 4096 & 32 & 8 & 128 & 128 & 8192 & $\tau=1$ & $\tau=1$ \\
    \midrule
    FullAttn & $\approx$14B & 296B,36B & 2M & 3e-4,3e-5 & 40 & 5120 & 40 & 8 & 128 & 128 & 10240 & 4096 & 32768 \\
    MoBA & $\approx$14B & 296B,36B & 2M & 3e-4,3e-5 & 40 & 5120 & 40 & 40 & 128 & 128 & 10240 & $K=1024$ & $K=1024$ \\
    DSA & $\approx$14B & 296B,36B & 2M & 3e-4,3e-5 & 40 & 5120 & 40 & 1 & 576 & 512 & 40960 & $K=256$ & $K=256$ \\
    \shortName{} & $\approx$14B & 296B,36B & 2M & 3e-4,3e-5 & 40 & 5120 & 40 & 8 & 128 & 128 & 10240 & $\tau=1$ & $\tau=1$ \\
    \bottomrule
    \end{tabular}
    }
\end{table}

The FLOP-axis uses the cumulative training FLOPs reported by Megatron-LM and therefore covers the complete model computation. To expose the attention-specific contribution to this statistic, we additionally record the work executed by each attention operator. We count one multiply-accumulate as two FLOPs and include the QK recomputation used by attention backward. Let $a_n=(n+1)/2$ be the mean number of legal causal keys per query for a sequence of length $n$, and let $\bar{K}=K-K(K-1)/(2n)$ be the corresponding mean retained work after causally clipping a fixed budget $K\le n$. The per-layer, per-token sequence-dependent attention terms are
\begin{align}
    C_{\mathrm{FullAttn}}
    &=8a_nhd_{qk}+6a_nhd_v, \\
    C_{\mathrm{MoBA}}
    &=8\bar{K}hd_{qk}+6\bar{K}hd_v+C_{\mathrm{router}}, \\
    C_{\mathrm{DSA}}
    &=8\bar{K}hd_{qk}+6\bar{K}hd_v+C_{\mathrm{indexer}}, \\
    C_{\shortName{}}
    &=4a_nhd_{qk}+4\bar{s}_bhd_{qk}
      +2\bar{s}_fhd_v+4\bar{s}_bhd_v.
    \label{eq:appendix:scaling_attention_flops}
\end{align}
For \shortName{}, $\bar{s}_f$ and $\bar{s}_b$ are the mean forward and backward post-score slots per query measured during training. The 4K stage records $(\bar{s}_f,\bar{s}_b)=(906.78,896.74)$, and the 32K stage records $(1{,}065.95,1{,}057.64)$. Its first term counts forward score discovery and backward QK recomputation over the complete causal region; the remaining terms count forward PV and the retained $\vdQ$, $\vdK$, $\vdP$, and $\vdV$ paths. $C_{\mathrm{router}}$ includes MoBA pooling, dense query-to-block routing scores, and Top-$K$ comparisons. $C_{\mathrm{indexer}}$ includes DSA's dense lightning-indexer score path, index transforms and quantization arithmetic, and Top-$K$ comparisons.

\subsection{Sparse Adaptation via Continued Training}
\label{sec:appendix:sparse_adaptation_setup}

We study adaptation at 32B through long-context continued training, using FullAttn as the dense reference. Both models are initialized from the same Qwen3-32B checkpoint and trained at sequence length 32,768 for 64B tokens with a batch size of 4M tokens and a constant learning rate of $1\times10^{-5}$. Both configurations use 64 layers, hidden size 5,120, 64 QO heads, 8 KV heads, and $d_{qk}=d_v=128$. FullAttn uses the complete causal context. \shortName{} uses the canonical tolerance $\tau=1$, retains complete causal QK score discovery, and records its distribution-determined post-score work.

\subsection{Benchmark Evaluation}
\label{sec:appendix:benchmark_setup}

We evaluate two groups of models. At 14B, we compare FullAttn, MoBA, DSA, and \shortName{} from the scaling-law study in Appendix~\ref{sec:appendix:scaling_laws_setup}. At 32B, we compare the FullAttn and \shortName{} models from the continued-training setup in Appendix~\ref{sec:appendix:sparse_adaptation_setup}.
For each task, we report the mean and standard deviation over five evaluation runs with seeds 0, 42, 233, 666, and 1234.
The standard benchmark suite covers knowledge, reasoning, and retrieval through MMLU~\citep{hendrycks2021measuring,lyu2024probabilitiesunveilingmisalignmentevaluating}, MMLU-Pro~\citep{wang2024mmluprorobustchallengingmultitask}, BBH~\citep{suzgun2023challenging}, HellaSwag~\citep{zellers2019hellaswag}, OBQA~\citep{mihaylov2018can}, WinoGrande~\citep{sakaguchi2021winogrande}, PIQA~\citep{bisk2020piqa}, GSM8K~\citep{cobbe2021training}, Hendrycks-Math~\citep{hendrycks2021measuringmathematicalproblemsolving}, ARC-C~\citep{clark2018think}, AGIEval~\citep{zhong2023agievalhumancentricbenchmarkevaluating}, GPQA-Diamond~\citep{rein2023gpqagraduatelevelgoogleproofqa}, and RULER~\citep{hsieh2024ruler}. We evaluate RULER at the native 32K context length and, after YaRN position extrapolation, at 128K. Benchmark prompts, shot counts, decoding settings, seeds, and metric aggregation are held fixed across attention variants within each model scale.

\section{Detailed Model-Level Results}
\label{sec:appendix:model_results}

Tables~\ref{table:knowledge_benchmark},~\ref{table:reasoning_benchmark}, and~\ref{table:retrieval_benchmark} report the per-task results summarized in Table~\ref{table:model_level_summary}.

\begin{table}[H]
    \centering
    \caption{
    \textbf{Knowledge Benchmark Results}.
    Knowledge evaluation of the scaling-law models at 14B and the continued-training models at 32B. \shortName{} produces averages comparable to FullAttn at both scales and remains close across individual tasks.
    }
    \vspace{-0.75em}
    \resizebox{\linewidth}{!}{
    \begin{tabular}{@{}lccccccccccccccc@{}}
    \toprule
    \sc{Model} & \sc{MMLU} & \sc{MMLU-Pro} & \sc{BBH} & \sc{HellaSwag} & \sc{OBQA} & \sc{WinoGrande} & \sc{Avg}
    \\
    & \sc{Gen|5-shot} & \sc{5-shot} & \sc{CoT|3-shot} & \sc{5-shot} & \sc{5-shot} & \sc{5-shot} &
    \\
    \midrule
    \multicolumn{8}{c}{\text{14B-Non-Thinking, temperature=0.6, top\_k=20, top\_p=0.95, max\_gen\_toks=32768, seed=[0, 42, 233, 666, 1234]}} \\
    \midrule
    FullAttn & \textbf{78.21$\pm$0.31} & \textbf{64.96$\pm$0.46} & \textbf{82.59$\pm$0.48} & \textbf{80.64$\pm$0.38} & 49.4$\pm2.17$ & \underline{78.16$\pm$1.14} & \underline{72.32} \\
    MoBA & 75.78$\pm$0.48 & 58.42$\pm$0.84 & 82.1$\pm$0.44 & 79.7$\pm$0.37 & \underline{49.49$\pm$2.07} & 74.57$\pm$1.47 & 70.01 \\
    DSA & 72.5$\pm$0.53 & 59.56$\pm$0.65 & 81.42$\pm$0.62 & 80.22$\pm$0.31 & 49.41$\pm$2.25 & 77.57$\pm$1.19 & 70.11 \\
    \shortName{} (ours) & \underline{78.14$\pm$0.31} & \underline{64.67$\pm$0.44} & \underline{82.54$\pm$0.48} & \underline{80.61$\pm$0.36} & \textbf{50.0$\pm$2.17} & \textbf{78.93$\pm$1.13} & \textbf{72.48} \\
    \midrule
    \multicolumn{8}{c}{\text{32B-Thinking, temperature=0.6, top\_k=20, top\_p=0.95, max\_gen\_toks=32768, seed=[0, 42, 233, 666, 1234]}} \\
    \midrule
    FullAttn & \underline{77.27$\pm$0.34} & \textbf{68.57$\pm$0.41} & \textbf{88.39$\pm$0.54} & \underline{83.4$\pm$0.37} & \textbf{55.0$\pm$2.23} & \textbf{81.14$\pm$1.1} & \underline{75.62} \\
    \shortName{} (ours) & \textbf{77.37$\pm$0.33} & \underline{68.33$\pm$0.41} & \underline{88.23$\pm$0.55} & \textbf{84.69$\pm$0.36} & \textbf{55.0$\pm$2.23} & \underline{80.87$\pm$1.13} & \textbf{75.75} \\
    \bottomrule
    \end{tabular}
  }
  \vspace{-1.0em}
  \label{table:knowledge_benchmark}
\end{table}

\paragraph{Knowledge performance.}
Table~\ref{table:knowledge_benchmark} shows that \shortName{} retains knowledge benchmark performance close to FullAttn, with average scores of 72.48 versus 72.32 at 14B and 75.75 versus 75.62 at 32B.
At 14B, the differences on MMLU, MMLU-Pro, BBH, and HellaSwag are all within 0.29 points.
At 32B, the largest difference is a 1.29-point gain on HellaSwag, while all other differences remain within 0.27 points.
The small aggregate gains therefore support comparable performance across the evaluated knowledge tasks rather than a uniform improvement.

\begin{table}[H]
  \centering
  \caption{
    \textbf{Reasoning Benchmark Results}.
    Reasoning evaluation of the scaling-law models at 14B and the continued-training models at 32B. \shortName{} matches FullAttn on average at 14B and 32B without a fixed retained-work budget.
  }
  \vspace{-0.75em}
  \resizebox{\linewidth}{!}{
    \begin{tabular}{@{}lccccccccccccccc@{}}
    \toprule
    \sc{Model} & \sc{PIQA} & \sc{GSM8K} & \sc{Hendrycks-Math} & \sc{ARC-C} & \sc{AGIEval} & \sc{GPQA-Diamond} & \sc{Avg}
    \\
    & \sc{5-shot} & \sc{CoT|4-shot} & \sc{CoT|4-shot} & \sc{0-shot} & \sc{0-shot} & \sc{CoT|3-shot} &
    \\
    \midrule
    \multicolumn{8}{c}{\text{14B-Non-Thinking, temperature=0.6, top\_k=20, top\_p=0.95, max\_gen\_toks=32768, seed=[0, 42, 233, 666, 1234]}} \\
    \midrule
    FullAttn & \textbf{81.63$\pm$0.86} & \underline{90.32$\pm$0.93} & \textbf{62.46$\pm$0.65} & \textbf{61.95$\pm$1.44} & \textbf{50.56$\pm$0.71} & 39.87$\pm$3.47 & \underline{64.46} \\
    MoBA & 77.54$\pm$0.94 & 87.53$\pm$1.02 & 38.67$\pm$1.39 & 61.35$\pm$0.63 & 46.5$\pm$0.81 & 36.87$\pm$3.52 & 58.07 \\
    DSA & 80.35$\pm$0.82 & 89.38$\pm$0.95 & 53.24$\pm$0.86 & 61.74$\pm$1.33 & 47.0$\pm$0.69 & \underline{40.53$\pm$3.86} & 62.04 \\
    \shortName{} (ours) & \underline{81.62$\pm$0.85} & \textbf{90.83$\pm$0.9} & \underline{62.34$\pm$0.64} & \underline{61.93$\pm$1.47} & \underline{50.54$\pm$0.71} & \textbf{40.67$\pm$3.32} & \textbf{64.66} \\
    \midrule
    \multicolumn{8}{c}{\text{32B-Thinking, temperature=0.6, top\_k=20, top\_p=0.95, max\_gen\_toks=32768, seed=[0, 42, 233, 666, 1234]}} \\
    \midrule
    FullAttn & \textbf{85.97$\pm$0.88} & \underline{93.33$\pm$0.69} & \textbf{94.94$\pm$0.82} & \underline{69.54$\pm$1.34} & \textbf{54.18$\pm$0.68} & \underline{56.06$\pm$3.54} & \underline{75.67} \\
    \shortName{} (ours) & \underline{85.94$\pm$0.88} & \textbf{95.45$\pm$0.57} & \underline{94.25$\pm$0.83} & \textbf{69.8$\pm$1.34} & \underline{53.87$\pm$0.66} & \textbf{57.29$\pm$2.89} & \textbf{76.1} \\
    \bottomrule
    \end{tabular}
  }
  \vspace{-1.0em}
  \label{table:reasoning_benchmark}
\end{table}

\paragraph{Reasoning performance.}
Table~\ref{table:reasoning_benchmark} shows comparable aggregate reasoning performance, with average scores of 64.66 versus 64.46 for FullAttn at 14B and 76.10 versus 75.67 at 32B.
At 14B, \shortName{} remains within 0.12 points of FullAttn on PIQA, Hendrycks-Math, ARC-C, and AGIEval, with higher scores on GSM8K and GPQA-Diamond.
At 32B, gains on GSM8K and GPQA-Diamond coexist with modest decreases on Hendrycks-Math and AGIEval.
These results indicate that the reduction in post-score computation preserves the evaluated reasoning performance overall, with task-specific differences.

\begin{table}[H]
  \centering
  \caption{
    \textbf{Retrieval Benchmark Results}.
    RULER evaluation at the native 32K context length and after YaRN extrapolation to 128K. \shortName{} remains comparable to FullAttn across both model scales and context lengths.
  }
  \vspace{-0.75em}
  \resizebox{\linewidth}{!}{
    \begin{tabular}{@{}lccccccccc@{}}
    \toprule
    \sc{Model} & \sc{NIAH-S} & \sc{NIAH-MK} & \sc{NIAH-MQ} & \sc{NIAH-MV} & \sc{RULER-VT} & \sc{RULER-CWE} & \sc{RULER-FWE} & \sc{RULER-QA} & \sc{Avg} \\
    \\
    \midrule
    \multicolumn{10}{c}{\text{14B Native 32K Sequence Length, seed=[0, 42, 233, 666, 1234]}} \\
    \midrule
    FullAttn & \textbf{100$\pm$0} & \textbf{98.0$\pm$0.65} & \textbf{98.8$\pm$0.68} & \textbf{95.5$\pm$0.72} & \textbf{99.44$\pm$0.2} & \textbf{75.7$\pm$0.69} & \underline{91.67$\pm$0.71} & \underline{56.28$\pm$1.2} & \underline{89.42} \\
    MoBA & 99.2$\pm$0.4 & 93.2$\pm$0.66 & 98.3$\pm$0.71 & 94.8$\pm$0.73 & \underline{99.08$\pm$0.33} & 66.8$\pm$0.99 & 91.3$\pm$0.72 & 52.02$\pm$1.64 & 86.85 \\
    DSA & \underline{99.8$\pm$0.2} & 95.2$\pm$0.69 & 97.8$\pm$0.79 & \underline{95.4$\pm$0.74} & 98.87$\pm$0.31 & 69.05$\pm$0.89 & 91.1$\pm$0.74 & 53.51$\pm$1.52 & 87.59 \\
    \shortName{} (ours) & \textbf{100$\pm$0} & \underline{97.5$\pm$0.68} & \underline{98.6$\pm$0.68} & 95.3$\pm$0.73 & 99.03$\pm$0.22 & \underline{74.32$\pm$0.74} & \textbf{93.76$\pm$0.59} & \textbf{57.05$\pm$1.18} & \textbf{89.45} \\
    \midrule
    \multicolumn{10}{c}{\text{32B Native 32K Sequence Length, seed=[0, 42, 233, 666, 1234]}} \\
    \midrule
    FullAttn & \textbf{100$\pm$0} & \textbf{99.87$\pm$0.09} & \underline{99.8$\pm$0.08} & \textbf{99.4$\pm$0.13} & \underline{99.64$\pm$0.18} & \textbf{85.18$\pm$0.51} & \underline{94.53$\pm$0.53} & \underline{63.2$\pm$0.93} & \underline{92.7} \\
    \shortName{} (ours) & \textbf{100$\pm$0} & \underline{99.53$\pm$0.11} & \textbf{99.95$\pm$0.05} & \underline{99.25$\pm$0.19} & \textbf{99.88$\pm$0.07} & \underline{83.3$\pm$0.54} & \textbf{96.33$\pm$0.81} & \textbf{63.44$\pm$1.02} & \textbf{92.71} \\
    \midrule
    \multicolumn{10}{c}{\text{14B YaRN 128K Sequence Length, seed=[0, 42, 233, 666, 1234]}} \\
    \midrule
    FullAttn & \underline{98.4$\pm$0.51} & \textbf{55.2$\pm$1.64} & \textbf{88.25$\pm$0.69} & 64.1$\pm$1.16 & \underline{93.04$\pm$0.63} & \textbf{5.76$\pm$0.46} & \textbf{84.26$\pm$0.79} & \textbf{37.73$\pm$2.15} & \textbf{65.84} \\
    MoBA & 92.2$\pm$0.83 & 51.6$\pm$1.85 & 80.3$\pm$0.87 & 62.05$\pm$1.3 & 90.48$\pm$0.71 & 0.02$\pm$0.06 & 79.93$\pm$0.95 & 24.78$\pm$2.63 & 60.17 \\
    DSA & 96.4$\pm$0.62 & 52.7$\pm$1.73 & 83.6$\pm$0.81 & \textbf{64.3$\pm$1.16} & 92.48$\pm$0.62 & 4.8$\pm$0.52 & 77.93$\pm$1.16 & 32.19$\pm$2.03 & 63.05 \\
    \shortName{} (ours) & \textbf{98.6$\pm$0.53} & \underline{54.6$\pm$1.67} & \underline{88.2$\pm$0.69} & \underline{64.2$\pm$1.14} & \textbf{93.08$\pm$0.63} & \underline{5.51$\pm$0.45} & \underline{84.17$\pm$0.81} & \underline{37.62$\pm$2.12} & \underline{65.75} \\
    \midrule
    \multicolumn{10}{c}{\text{32B YaRN 128K Sequence Length, seed=[0, 42, 233, 666, 1234]}} \\
    \midrule
    FullAttn & \textbf{99.8$\pm$0.12} & \textbf{86.4$\pm$1.35} & \underline{95.7$\pm$0.56} & \textbf{88.0$\pm$0.79} & \textbf{98.76$\pm$0.22} & \underline{45.68$\pm$1.03} & \underline{93.27$\pm$1.25} & \underline{48.61$\pm$2.17} & \underline{82.03} \\
    \shortName{} (ours) & \underline{99.65$\pm$0.13} & \underline{86.3$\pm$1.35} & \textbf{95.9$\pm$0.66} & \underline{87.75$\pm$0.79} & \underline{98.69$\pm$0.23} & \textbf{45.8$\pm$1.04} & \textbf{96.33$\pm$0.81} & \textbf{50.03$\pm$1.97} & \textbf{82.56} \\
    \bottomrule
    \end{tabular}
  }
  \vspace{-1.0em}
  \label{table:retrieval_benchmark}
\end{table}

\paragraph{Long-context performance.}
Table~\ref{table:retrieval_benchmark} shows that \shortName{} closely tracks FullAttn at both context lengths.
At native 32K, their average scores differ by only 0.03 points at 14B and 0.01 points at 32B.
Both methods obtain lower aggregate scores under YaRN extrapolation to 128K, but remain close to each other: \shortName{} scores 65.75 versus 65.84 at 14B and 82.56 versus 82.03 at 32B.
At 14B and 128K, every reported per-task difference is within 0.60 points; at 32B, the higher aggregate score includes gains on RULER-FWE and RULER-QA.
These results support comparable long-context retrieval performance under the evaluated extrapolation setting.